%% file: main.tex
\documentclass[11pt]{article}
\usepackage{fullpage}
\usepackage{amsmath, amssymb, amsfonts, amsthm}
\usepackage{natbib}
\usepackage{titlesec}

\usepackage[capitalize,noabbrev]{cleveref}

\usepackage{mathtools}
\usepackage{algorithm}
\usepackage{algorithmic}
\usepackage{subfig}

\input{macros_icml}

\theoremstyle{plain}
\newtheorem{theorem}{Theorem}[section]

\newtheorem{lemma}[theorem]{Lemma}

\theoremstyle{definition}

\theoremstyle{remark}

\newcommand{\basic}{$\mathsf{VarAlloc}$}
\newcommand{\gena}{$\mathsf{CorrVarAlloc}$}
\newcommand{\genb}{$\mathsf{GraphVarAlloc}$}
\newcommand{\genbcorr}{$\mathsf{CorrGraphVarAlloc}$}
\newcommand{\poly}{\text{poly}}

\title{Allocating Variance to Maximize Expectation}
\usepackage{times}
\author{Renato Purita Paes Leme\thanks{Google inc.},~~Cliff Stein\thanks{Google inc. and Columbia U.},~~Yifeng Teng\thanks{Google inc.},~~Pratik Worah\thanks{Google inc., pworah@google.com}
}
\begin{document}
\maketitle

\begin{abstract}%
We design efficient approximation algorithms for maximizing the expectation of the supremum of families of Gaussian random variables. In particular, let $\opt:=\max_{\sigma_1,\cdots,\sigma_n}\E\sum_{j=1}^{m}\max_{i\in S_j} X_i$, where $X_i$ are Gaussian, $S_j\subset[n]$ and $\sum_i\sigma_i^2=1$, then our theoretical results include:
\begin{itemize}
    \item We characterize the optimal variance allocation -- it concentrates on a small subset of variables as $|S_j|$ increases,
    \item A polynomial time approximation scheme (PTAS) for computing $\opt$ when $m=1$, and
    \item An $O(\log n)$ approximation algorithm for computing $\opt$ for general $m>1$.
\end{itemize}

Such expectation maximization problems occur in diverse applications, ranging from utility maximization in auctions markets to learning mixture models in quantitative genetics.
\end{abstract}


\section{Introduction}

The total accuracy of a machine learning (ML) model is constrained by several factors such as the available training data,  model architecture, and training algorithms, and while it may not be possible to improve the average model accuracy past a certain point in a given setup, one still has freedom to tune the model to trade-off accuracy in one subset of examples for accuracy in another subset. One approach is to reweigh the data~\citep{dataset-book, cortes1, cortes2}. Consider the well-known problem of predicting a protein's 3D structure, using ML algorithms to learn from labeled examples. One may want to tune the model to be more accurate for predicting the structure of cell surface proteins for designing drugs that target surface proteins~\citep{rfdiff}; but one may want to tune the model to be more accurate for predicting the structure of DNA binding enzymes~\citep{opencrispr}, for designing gene editing therapies. In principle, both can be achieved, using the same labeled training dataset, by reweighing examples appropriately. One critical question is: how much to up-weight or down-weight any given training example? 

For an ML algorithm, the accuracy of predictions in a neighborhood of the data-space may be approximated by the variance of predictions in that neighborhood. So, making predictions more accurate in a given neighborhood, by up-weighting training data-points from that neighborhood, should closely correspond to lowering the {\em local variance} in a neighborhood.\footnote{
In fact, it is easy to explicitly relate reweighing the loss function to the variance allocation problem, when training linear regression models (see Subsection~\ref{sbs:mot}).} If one can formulate the (weighted) loss function as a function of corresponding local variances, then the problem becomes one of computing optimal variance allocation to minimize the (training or validation) loss, or equivalently maximize a reward. Thus, in this paper, we focus on designing efficient algorithms for how to optimally allocate variance for a simple class of objective functions involving expectation maximization.



The organization of the paper is as follows. In Subsection~\ref{sbs:form}, we formulate the abstract optimization problems formally; in Subsection~\ref{sbs:results}, we state our main structural lemma and theorems, which correspond to the proofs of correctness of the approximation algorithm mentioned above; in Subsection~\ref{sbs:sims}, we provide figures (Figures~\ref{payoff:fig} and~\ref{fig:comp}) that show general trends in the solution as the dependency among the Gaussians increases; in Subsection~\ref{sbs:mot} we discuss applications to machine learning and auctions; in Subsection~\ref{sbs:related}, we discuss related works, especially how our results compare to prior extensive work in this area~\citep{talagrand}; and finally in Section~\ref{sec:overview} we describe an overview our theorems, and prove Lemma~\ref{lem:eps-contribution}, which contains a modified chaining argument that may be of independent interest. The appendix contains the remaining theorems and proofs.

\subsection{Problem formulation}\label{sbs:form}


We are motivated by trade-offs in the accuracy of ML models, but we defer any motivational details to Subsection~\ref{sbs:mot}, and formulate our main question as a stochastic optimization problem. We start by defining its basic version (dense and independent) and later consider correlated and sparse variants.


\noindent {\em Variance Allocation Problem:} We are given  $n$ independent Gaussian random variables $(X_1,\hdots,X_n)$ with means $\mu_1, \hdots, \mu_n$. 
Our goal is to assign values to the variances, $\sigma_1^2, \hdots, \sigma_n^2$, so as to maximize the expected value of the maximum $X_i$ subject to the constraint that the total variance is bounded, that is 
  the variance vector satisfies$\sum_{i=1}^{n}\sigma_i^2 = 1$.  We  let $\opt$ denote the optimal objective in this setting, and summarize the problem as (\basic):
\begin{eqnarray}
\opt&=&\max_{\sigma_1,...,\sigma_n}\ \E\max_{i\in [n]} X_i\\
\textrm{subject to}& &\sum_{i=1}^{n}\sigma_i^2 = 1. \label{eqn:independent-constraint}
\end{eqnarray}

\noindent {\em Correlated Variance Allocation:} The above formulation (\basic) uses independent Gaussians,  but we also consider the non-independent case, where we have covariances. We are again given means $\mu_1, \hdots, \mu_n$ and are asked to choose a covariance matrix $\Sigma$ to maximize the expected maximum of $(X_1,X_2,\cdots,X_n)\sim\N(\mu,\Sigma)$ such that the total variance is bounded $\sum_{i=1}^{n}\Sigma_{ii} = 1$. Thus we get (\gena):
\begin{eqnarray}
\opt&=&\max_{\Sigma}\ \E\max_{i\in [n]} X_i\\
\textrm{subject to}& &\sum_{i=1}^{n}\Sigma_{ii} = 1 \\ \label{eqn:correlated-constraint}
& &  \Sigma \mbox{ is PSD}
\end{eqnarray}

\noindent {\em Graph Variance Allocation:} 
We generalize (\basic) in a different way by allowing grouping among variables.
We are now given, in addition to the $n$ independent and normally distributed variables, $m$ subsets $S_j\subseteq[n]$. We denote the contribution of the set $S_j $ by: $\max_{i\in S_j}X_i$, and our objective is to allocate the variance with the same constraint \eqref{eqn:independent-constraint} and maximize the total expected contribution of the $m$ sets. More formally (\genb):

\begin{eqnarray}
\opt&=&\max_{\sigma_1,\cdots,\sigma_n}\E\sum_{j=1}^{m}\max_{i\in S_j} X_i\\
\textrm{subject to}& &\sum_{i=1}^{n}\sigma_i^2= 1.\label{eqn:multi-sets-constraint}
\end{eqnarray}

\noindent {\em Correlated Graph Variance Allocation:} 
We can also generalize (\genb) to use correlated Gaussians.
We are again given means $\mu_1, \hdots, \mu_n$ and $m$ subsets $S_j\subseteq[n]$, and want to choose a covariance matrix $\Sigma$ to maximize the sum of the contribution $\max_{i\in S_j}X_i$ of each set $S_j$ such that the total variance is bounded $\sum_{i=1}^{n}\Sigma_{ii} = 1$. More formally (\genbcorr):

\begin{eqnarray}
\opt&=&\max_{\sigma_1,\cdots,\sigma_n}\E\sum_{j=1}^{m}\max_{i\in S_j} X_i\\
& &  \Sigma \mbox{ is PSD}.
\end{eqnarray}

\noindent Note that, when $\Sigma$ is assumed diagonal:\begin{itemize}
\item \gena~is equivalent to \basic, and 
\item \genbcorr~is equivalent to \genb; 
\end{itemize} 
and when $m=1$, \genb~is equivalent to \basic.
In each case, we want to compute the optimal $\boldsigma$s or $\Sigma$ using an efficient approximation algorithm. We will use $\obj$ to denote the objective function studied in each case: $\obj=\E\max_{i\in[n]} X_i$ for \basic~and \gena, and $\obj=\E\sum_{i=1}^{m}\max_{i\in S_j} X_i$ for \genb.

\subsection{Algorithmic Results}\label{sbs:results}

We state our main technical results here.  Our algorithms for \basic \ and \gena \ will be stated as additive approximations, while for \genb, it is stated a multiplicative approximation. For \basic~and \gena, the additive PTAS is at least as strong as a multiplicative PTAS, since $\opt$ is at least $\Omega(1)$: when setting $\sigma_1=1$ and $\sigma_i=0$ for $i\geq 1$, $\E\max_i X_i=\sqrt{\frac{2}{\pi}}$ for $n\geq 2$ and $0$ for $n=1$.  



\begin{restatable}{theorem}{thmptasindependent}\label{thm:PTAS-anymean}
Given an input to \basic \  with 
 $\mu_1,\cdots,\mu_n\geq0$ and constant $\eps>0$, there exists an algorithm with running time polynomial in $n$, that computes a variance vector $(\hatsigma_1,\cdots,\hatsigma_n)$ such that $\E\max_{i\in [n]} X_i\geq \opt-\eps$.

 
\end{restatable}

\begin{restatable}{theorem}{thmptascorrelated}\label{thm:PTAS-correlated}
Given an input to \gena \ and constant $\eps>0$, there exists an algorithm with running time polynomial in $n$, that computes a covariance matrix $\widehat{\Sigma}$ such that for $(X_1,\cdots,X_n)\sim\N(0,\widehat{\Sigma})$, $\E\max_{i\in [n]} X_i\geq \opt-\eps$.
\end{restatable}

\begin{restatable}{theorem}{thmmultiplesets}\label{thm:polylog-approx}
Given an input to \genb, there exists an algorithm with a running time polynomial in $n$  that computes a variance vector $(\hatsigma_1,\cdots,\hatsigma_n)$ such that
$\E\sum_{j=1}^{m}\max_{i\in S_j} X_i\geq \Omega\left(\frac{1}{\log n}\right)\opt$.
\end{restatable}

\begin{restatable}{theorem}{thmmultiplesetscorr}\label{thm:polylog-approx-corr}
Given an input to \genbcorr, there exists an algorithm with running time polynomial in $n$  that computes a covariance matrix $\widehat{\Sigma}$ such that for $(X_1,\cdots,X_n)\sim\N(0,\widehat{\Sigma})$,
$\E\sum_{j=1}^{m}\max_{i\in S_j} X_i\geq \Omega\left(\frac{1}{\log n}\right)\opt$.
\end{restatable}

The main proof ideas and algorithms are discussed in Section~\ref{sec:overview}. For complete proofs,
the optimization problem with a single set is discussed in Section~\ref{sec:singleset}. 
The proof of Theorem~\ref{thm:PTAS-anymean} is discussed in Section~\ref{subsec:independent}. The proof of Theorem~\ref{thm:PTAS-correlated} is discussed in Section~\ref{subsec:correlated}. The optimization problem with multiple sets is discussed in Section~\ref{sec:multipleset}. The proofs of Theorem~\ref{thm:polylog-approx} and Theorem~\ref{thm:polylog-approx-corr} are discussed in Section~\ref{subsec:multiple}.

\subsection{Properties of optimal solution}\label{sbs:sims}

In order to understand the structure of the optimal variance allocations, we run Monte-Carlo simulations on the Erd\'{o}s-Renyi random graphs and obtain the following plots (Figures~\ref{payoff:fig} and~\ref{fig:comp}) that characterize the optimal value ($\opt$) and the corresponding variance allocation for \genb. We found $\opt$ to be concave, and the corresponding allocation to be concentrated on a few variables, as the density of the instance increased -- a somewhat counterintuitive pattern -- even for small random graphs (see Figures~\ref{payoff:fig} and~\ref{fig:comp}):
\begin{itemize}
\item {\em Concavity:} The solution value $\opt$ appears to be concave as a function of $p=\frac{|S_j|}{n}$, irrespective of whether the Gaussian random variables are restricted to be independent, positive or negatively correlated (see Figure~\ref{payoff:fig}). 
The concavity is formally verified in Theorem~\ref{thm:per-set-submodular}, which shows that the objective for \genb, as a function of $p$, here $p=\frac{|S_j|}{n}$, is concave. The key to the proof is the submodularity of the objective function.
\item {\em Concentration:} The variances tend to concentrate, i.e., they are supported on a smaller subset of variables with increasing $p$, where $p=\frac{|S_j|}{n}$ (see Figure~\ref{fig:comp}(a,b,c,d)). This concentration is formally verified in Theorem~\ref{thm:random-graph-ordering}, which shows that only $\Theta(\frac{1}{p})$ variables are allocated a variance $\Omega(p)$ in the optimal allocation, i.e., most variances are small in the optimal allocation as $p$ increases. The proof of this theorem relies on Lemma~\ref{lem:eps-contribution}, which is our key structural lemma, and also is used in the proof of correctness of our approximation algorithms (c.f. Subsection~\ref{sbs:results}).
\end{itemize}
\begin{figure}[htb]
    \centering
  \includegraphics[scale=0.12]{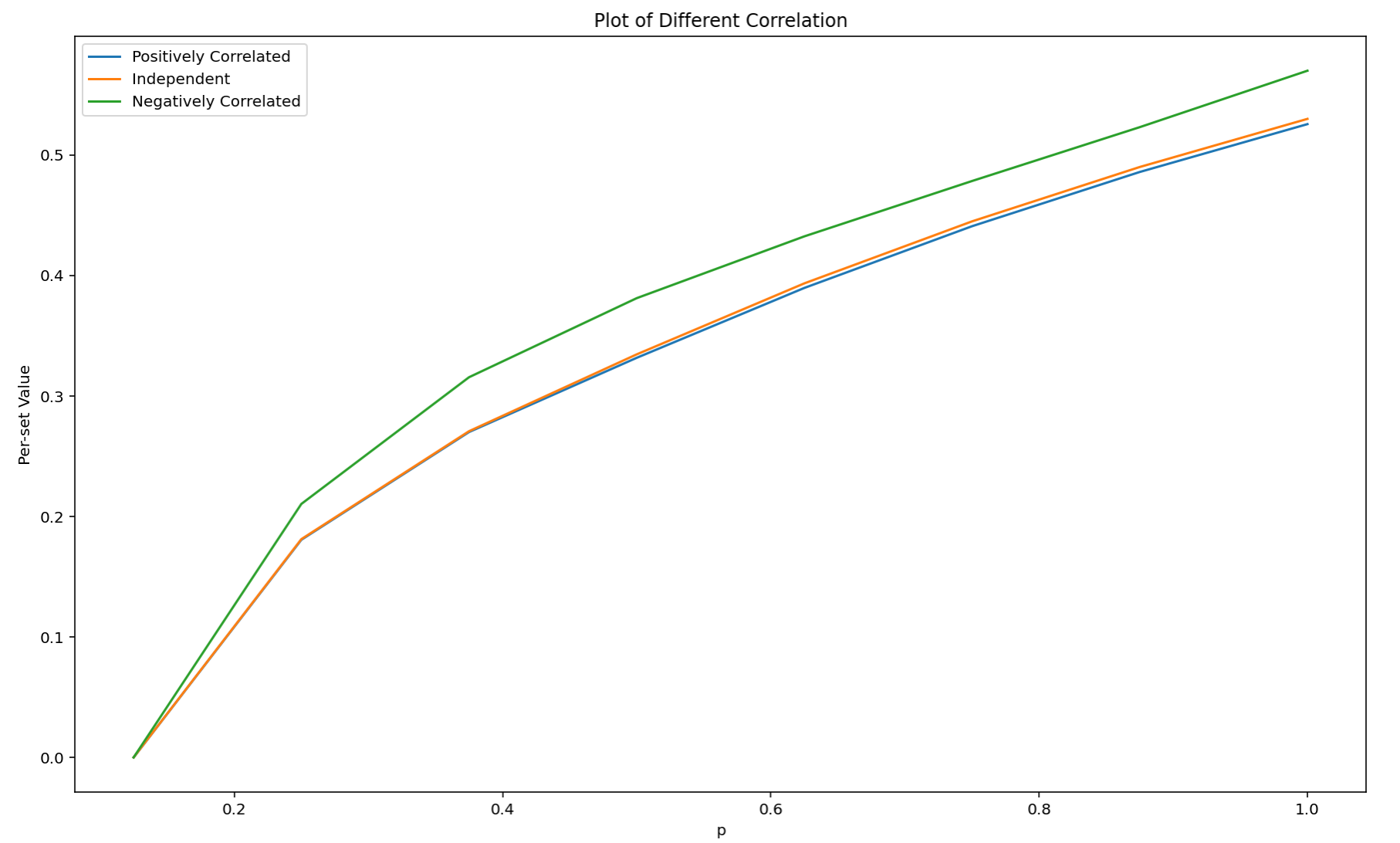}
  \caption{Concavity of optimal objective value per set (i.e. $\frac{1}{m}\opt$) when the Gaussians are restricted to be independent, be positively, or be negatively correlated in \genb, as a function of $p$ for Erd\'{o}s-Renyi graphs with $p\in\{\frac{1}{8},\frac{2}{8},...,\frac{8}{8}\}$ and $n=8$. To improve the simulation speed the positive and negative correlations are limited to block-diagonal matrices with block sizes $2\times2$, and $\Sigma_{2i+1,2i+2}=\pm\sqrt{\Sigma_{2i+1,2i+1}\cdot\Sigma_{2i+2,2i+2}}$. See also Theorem~\ref{thm:per-set-submodular} for the discussion of the independent case.}\label{payoff:fig}
\end{figure}
\begin{figure}[htb!]
    \centering
    \subfloat[][Optimal variance allocation for non-zero mean independent Gaussians, where the mean vector is $(0,0,0,0,0,0,1,1)$]{\includegraphics[scale=0.08]{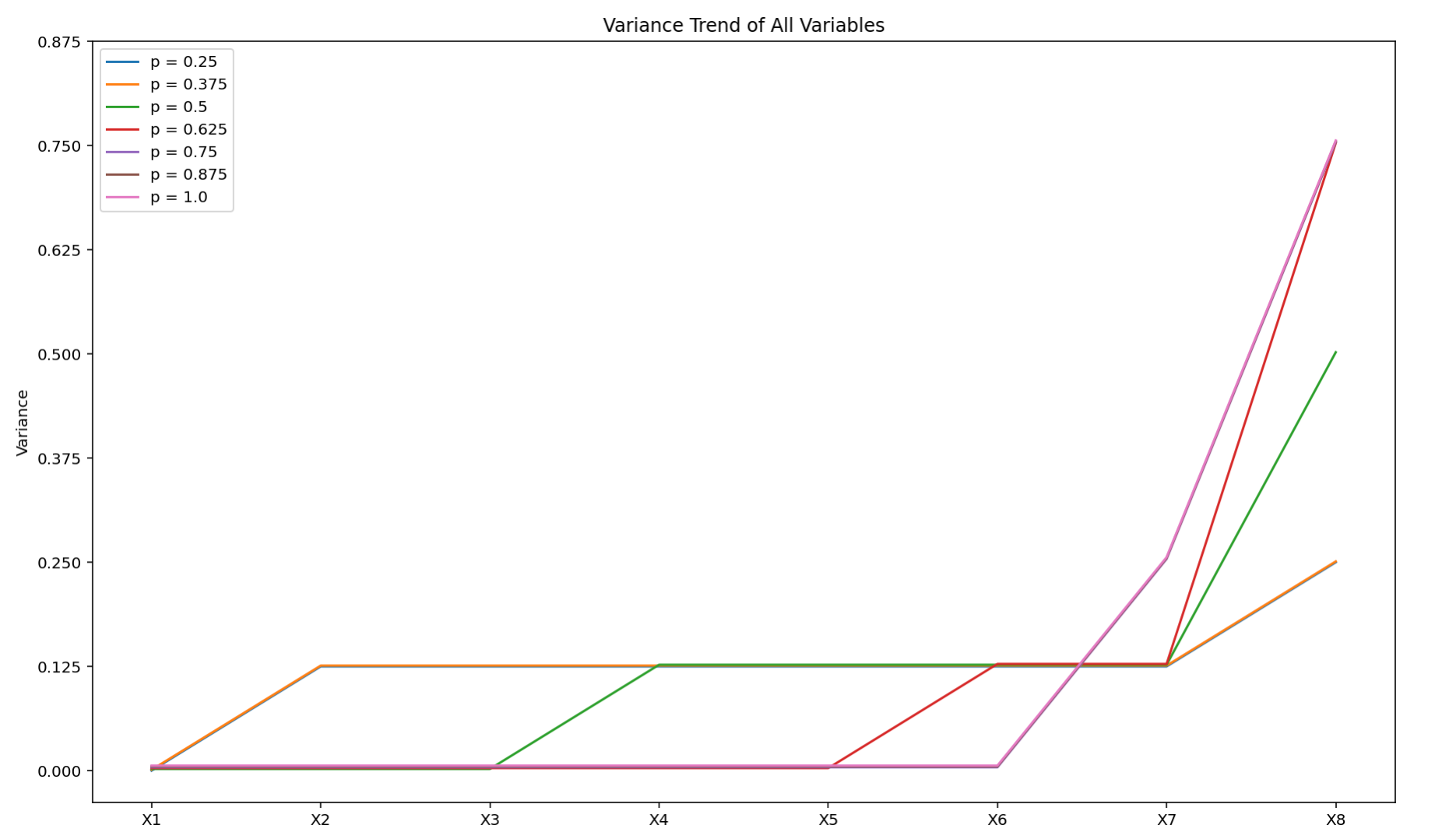}}\qquad
    \subfloat[][Optimal variance allocation for zero mean independent Gaussians]{\includegraphics[scale=0.08]{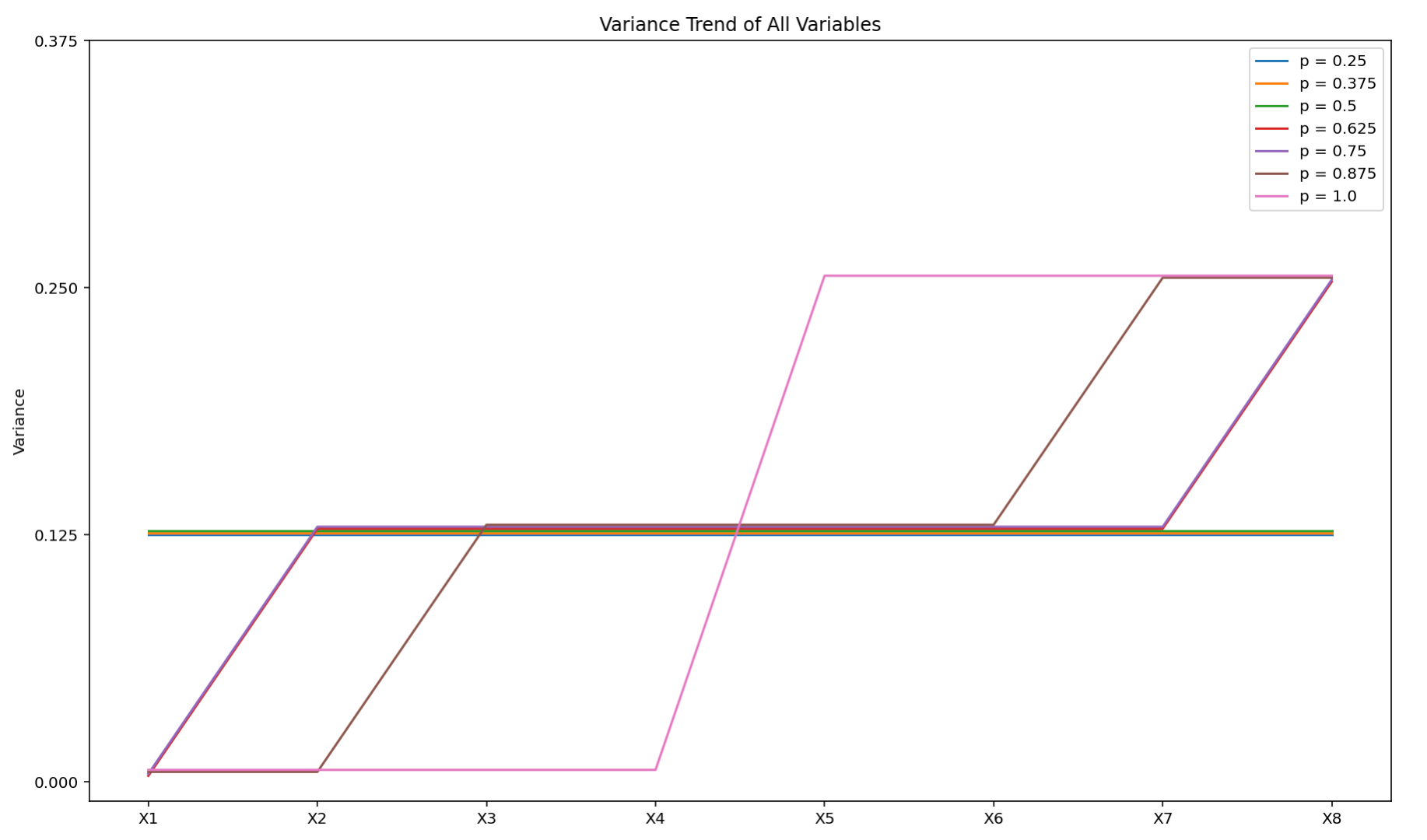}}\\
    \subfloat[][Optimal variance allocation for zero mean positively correlated Gaussians]{\includegraphics[scale=0.08]{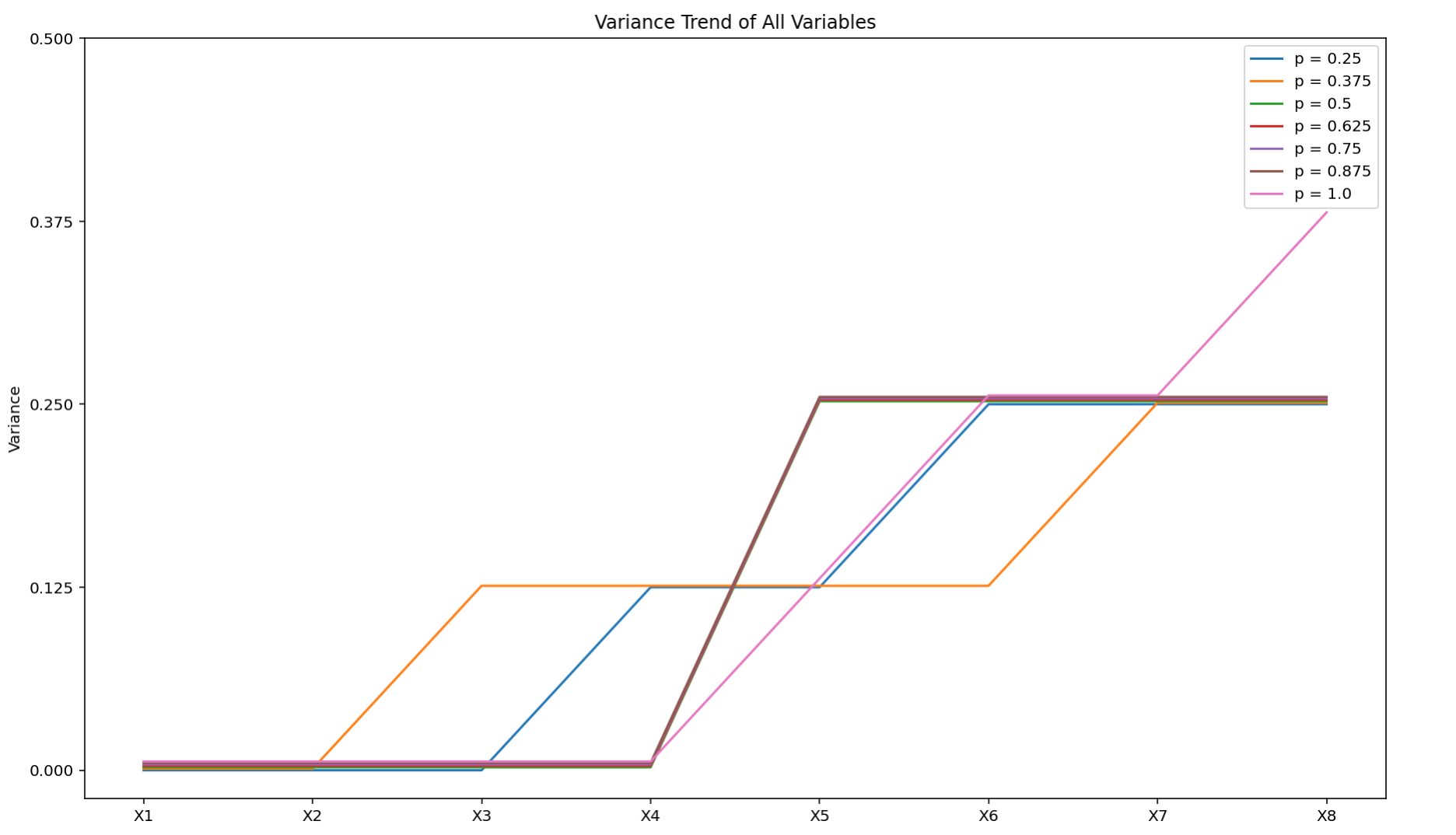}}\qquad
    \subfloat[][Optimal variance allocation for zero mean negatively correlated Gaussians]{\includegraphics[scale=0.08]{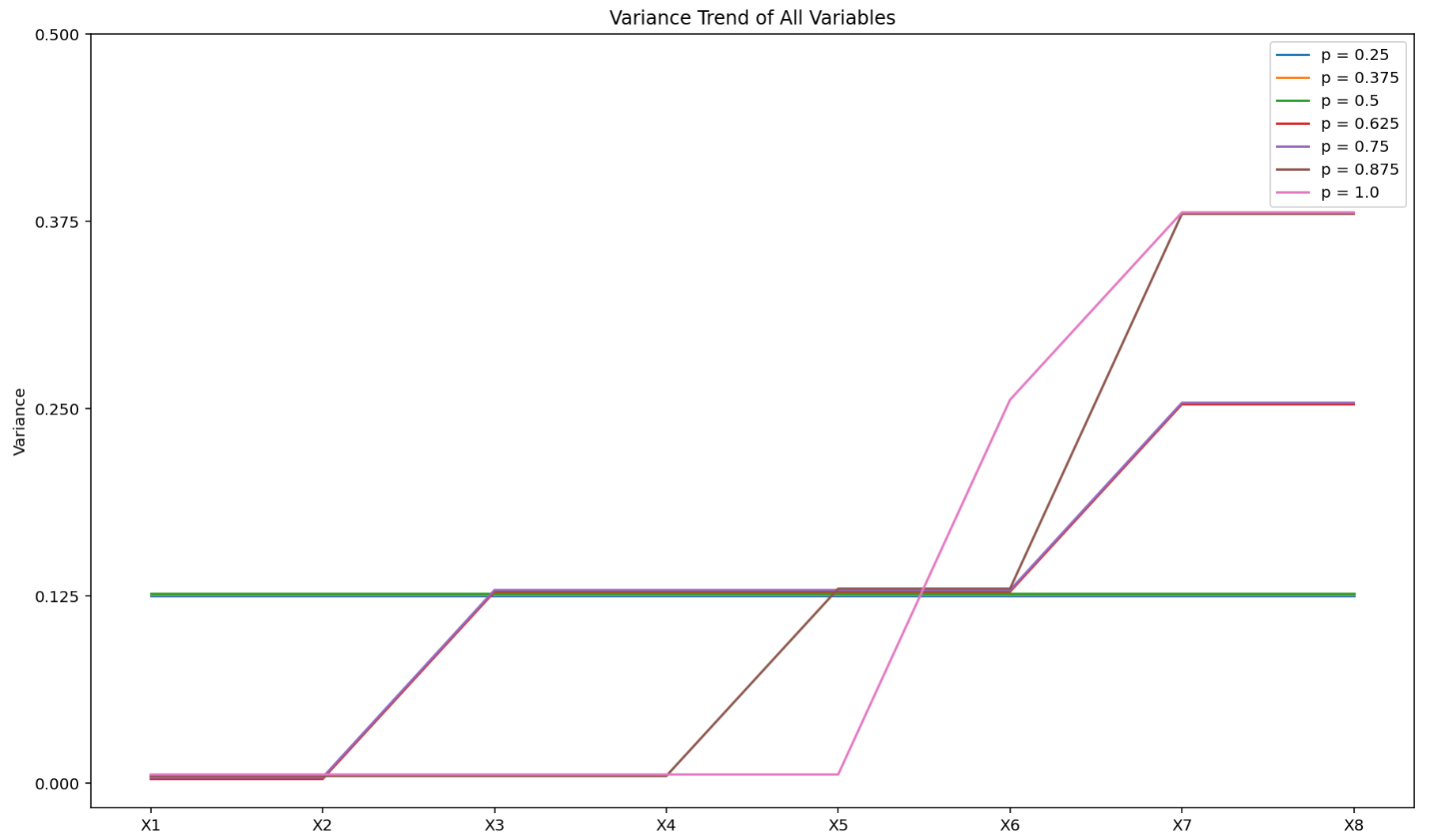}}
\caption{Simulations for Erd\'{o}s-Renyi graphs $G(n,p)$ for 
\genb~($n$ is held constant). $x$-axis corresponds to the $n$ variances (sorted by increasing values) and $y$-axis corresponds to their optimal values in allocation. Note the increasingly concentrated variance allocation with increasing $p$. See also Theorem~\ref{thm:random-graph-ordering}.}
\label{fig:comp}
\end{figure}


\begin{restatable}{theorem}{thmpersetsubmodular}\label{thm:per-set-submodular}
  Let $I_k$ be an instance for \genb\ with all $\binom{n}{k}$ sets of size $k$: for every $1\leq i_1<i_2<\cdots<i_k\leq n$, there is an set $S_j=\{i_1,\cdots,i_k\}$ in this instance. Let $f(k)=\frac{1}{\binom{n}{k}}\opt_{I_k}$ be the per-set contribution to the optimal objective in instance $I_k$. Then $g(k)$ is a concave function.  
\end{restatable}

\begin{restatable}{theorem}{thmrandomgraphordering}\label{thm:random-graph-ordering}
For any $\delta>0$ and a random instance of \genb with $n,m\to\infty$, with probability $p >1-\delta$, there are $\Theta(\frac{1}{p})$ variables allocated variance $\Omega(p)$.
\end{restatable}

\subsection{Motivating applications}\label{sbs:mot}
In this subsection, we outline some motivating applications. The theoretically inclined reader may skip this subsection. It is well-known that the study of bounds on the maximum value of stochastic processes $\E[\max_t X_t]$ has important applications, both practical ('What is the maximum magnitude of an earthquake in this region?' or 'What is the maximum level a river will rise?') as well as theoretical ones such a relation between properties of metrics spaces and bounds on stochastic processes. Those are discussed in detail in \cite{talagrand}.

\paragraph{Stochastic Gradient Descent}
The study of upper bounds on the expected behavior of stochastic processes also plays a vital role in ML. The stochastic gradient descent procedure (SGD) can be linearized near the local minima\footnote{See also the literature on wide neural networks as Gaussian processes~\citep{wwnngp}.} and viewed in the limit of small step size as an Ornstein Uhlenbeck (OU) process, which is a
Ito diffusion of the form:
\begin{equation}\label{eqn:ou}
    d\vec{X}(t) = \Phi(\vec{\mu}-\vec{X}(t))dt+\sigma dW(t),
\end{equation}
where $\vec{X},\vec{\mu}\in\mathbb{R}^n,\Phi,\sigma\in\mathbb{R}^{n\times n}$ and $W(t)$ is $n$-dimensional Standard Brownian Motion (see~\cite{oks} for details). When $\Phi$ is a diagonal matrix, it corresponds to changing the relative weight of different features. The long-term distribution of $\vec{X}$ (obtained after many steps of SGD), can be shown to be a Gaussian with variance $\Sigma\Sigma^T:=\Phi^{-1}\sigma\sigma^T$ whenever\ $\Sigma,\sigma,$ and $\Phi$ are diagonal. 

\paragraph{Computational biology}
Expectation maximization is a common objective used for estimating parameters of mixture models in quantitative biology applications~\citep{qb1,qb2,qb3,em1}. Whenever the components $X_i$ represent the likelihood of different events predicted by the ML system optimized by SGD, then the objective function $\E[\max X_i]$ corresponds to the maximum likelihood associated with the prediction (a common proxy for the log-likelihood obtained by replacing the softmax by a max). From that perspective, the Variance Allocation problem is a way of studying how the maximum likelihood changes based on changing the relative weight of the features $\Phi$ which are in one-to-one correspondence with the covariance matrix $\Sigma$.

\paragraph{Auction Optimization} 
A recent trend in the intersection of auction theory and ML \citep{bergemann2022calibrated,ChenLXZ24} is to maximize utility
by making the predictions that are fed into the auction more accurate (as opposed of optimizing auction rules directly). Previous papers in that line of work optimize the auction by introducing complex correlations between the predictions associated with different auction participants. From this perspective, the Variance Allocation Problem can be seen as a way to optimize auctions by changing the accuracy/variance of the prediction associated with different auction participants.

To make this precise, consider the setting of an advertising auction where the platform chooses the winner based on the product between their bid and a predicted click-through-rate and assume that the prediction system has the property that the limit distributions are Gaussian with variances controlled by the weights given to the different features, as in the SGD-example above. If the platform is using a first price auction (which is the common practice in display ads), the maximum $\E[\max_i X_i]$ corresponds to the revenue of the auction platform. From that perspective, the Variance Allocation problem is revenue optimization where the decision variable is the relative weight allocated to each feature in the prediction model.

\subsection{Related work}\label{sbs:related}
Talagrand~\citep{talagrand} has extensively studied problems of the form $\mathbb{E}[\sup_{t\in T}\{X_t\}]$, where $X_t$ are Gaussian. However, we are not aware of any of his work for the corresponding problems \basic, \gena, \genb, or \genbcorr, where we have (1) a budget on the total variance of the underlying Gaussian variables, and (2) the variables may have non-zero means. That said, our main structural lemma (Lemma~\ref{lem:eps-contribution}) does rely on a chaining argument (as do many other results in this area), which has been discussed thoroughly in Chapter 2 of~\cite{talagrand}. However, a direct application of methods in Chapter 2 of~\cite{talagrand} would give an upper bound of $O(\eps\sqrt{\log n})$ as opposed to $O(\eps\sqrt{\log \frac{1}{\eps}})$ (cf. Lemma~\ref{lem:eps-contribution}).


Our approximation algorithms corresponding to Theorems~\ref{thm:PTAS-correlated} and~\ref{thm:polylog-approx} rely crucially on Lemma~\ref{lem:eps-contribution}. We are able to obtain a PTAS because Lemma~\ref{lem:eps-contribution} implies that most of the variables will be set to zero in the optimal solution (a version of that result is in Theorem~\ref{thm:random-graph-ordering}). The logarithmic approximation algorithm on the other hand is from a reduction to submodular maximization, which is NP-hard but admits an efficient approximation algorithm~\citep{nemhauser1978analysis}. 
As an aside, it's worth noting that many variants of submodular maximization have been studied before \citep{Vondrak13}, including for non-monotone submodular functions \citep{BuchbinderF18}, monotone submodular functions \citep{FeigeMV11}, and submodular maximization with a matroid constraint \citep{BuchbinderF24}.

\section{Overview of technical results and proofs}\label{sec:overview}
In this section we provide an overview of the proofs of Theorems~\ref{thm:PTAS-anymean},~\ref{thm:PTAS-correlated} and~\ref{thm:polylog-approx} from Section~\ref{sbs:results}. The proof of Lemma~\ref{lem:eps-contribution}, which may be of independent interest, is presented in this section as well. The remaining detailed proofs are deferred to the appendix.

\subsection{Expectation maximization for a single set with independent variables}

As the optimization problem underlying \basic, \gena, and \genb~are non-convex, even when there is a single set, it is hard to find an efficient algorithm to solve the problem exactly. Although the solution space of the problem has $n$ dimensions for the independent setting and $n^2$ dimensions for joint Gaussian variables, it is still possible to find a polynomial-time approximation scheme, either additive or multiplicative.

The main intuition for the PTAS is as follows. We want to reduce the search space of near-optimal solutions. Firstly, all random variables with variance at most $\eps^2$ can only contribute to $\tilde{O}(\eps)$ in the optimization objective.\footnote{$\tilde{O}(\cdot)$ means the standard big O notation ignoring any poly-logarithmic factor.} Such an observation is formally shown by Lemma~\ref{lem:eps-contribution}. Therefore, we only need to consider solutions with positive variances of the variables being large enough $(\geq\eps^2)$, and there can only be $O(\eps^{-2})$ such variables, which is a constant for any fixed $\eps$. 
\begin{restatable}{lemma}{lemepscontribution}
\label{lem:eps-contribution}
For any $\eps\in(0,1)$ and positive integer $m$, let $m$ zero-mean Gaussian random variables $(Y_1,Y_2,\cdots,Y_n)\sim \N(0,\Sigma)$ satisfy $0\leq \Sigma_{11},\Sigma_{22},\cdots,\Sigma_{mm}\leq \eps^2$ and $\sum_{i}\Sigma_{ii}\leq 1$, we have $\E\max(0,\max_{i\in[m]}Y_i)=O(\eps\sqrt{\ln\frac{1}{\eps}})$.
\end{restatable}

We observe that a weaker bound of $O(\eps\sqrt{\log n})$ follows from directly applying Talagrand's chaining arguments in Chapter 2 of~\cite{talagrand}.  The weaker bound leads to a QPTAS (i.e. it runs in $n^{\text{polylog}(n)}$ time) since we can only guarantee $\text{polylog}(n)$ variables with negligible variance. Here we are able to obtain the stronger bound of $O(\eps\sqrt{\log \frac{1}{\eps}})$, which is essential to the PTAS. 

Our proof begins by following the approach in \cite{talagrand} of approximating the maximum as a sum over maximums of smaller subsets of random variables. However, crucially, since the total variance is bounded, we are able to prove tighter upper bounds on the contribution of each group (Equation~\ref{eqn:crux}).  The upshot of these improved bounds is to 
replace a $\log n$ factor by a $\log\frac{1}{\eps}$. 


\begin{proof}[Proof of Lemma~\ref{lem:eps-contribution}]
Firstly, we group these random variables into groups $G_1,G_2,\cdots$ according to their variance. Each group $G_j$ contains all random variables $Y_i$ with variance $\Sigma_{ii}$ satisfying $2^{-j}\eps<\sqrt{\Sigma_{ii}}\leq 2^{1-j}\eps$. For variables in group $G_j$, let $Z_j=\max(0,\max_{Y_i\in Z_j}Y_i)$ be the maximum of all variables in this group (lower-bounded by 0). Let $m_j=|G_j|$ be the number of variables in this group. We have for any $t\geq 0$,

$$e^{t\E Z_j} \leq \E[e^{tZ_j}] =\E[e^{t\max(0,\max_{Y_i\in G_j}Y_i)}]  =\E\max\left(e^{t\cdot 0},\max_{Y_i\in G_j}e^{Y_i}\right)$$
by the convexity of the exponential function and Jensen's inequality and the definition of $Z_j$. Now, since the expectation of maximum is at least the sum we have:

$$e^{t\E Z_j} \leq 1+\sum_{Y_i\in G_j} \E[e^{tY_i}] = 1+\sum_{Y_i\in G_j} e^{t^2\Sigma_{ii}/2} \leq 1+m_je^{2^{1-2j}t^2\eps^2} $$
where the equality follows by the expectation of the exponential of normal variables and the last inequality by $\Sigma_{ii}\leq 2^{2-2j}\eps$ for any $Y_i\in G_j$. Taking the logarithm of both sides, we get

\begin{equation*}
t\E Z_j\leq \ln(1+m_je^{2^{1-2j}t^2\eps^2})\leq \ln((1+m_j)e^{2^{1-2j}t^2\eps^2})=\ln(1+m_j)+2^{1-2j}t^2\eps^2.
\end{equation*}
By taking $t=\frac{2^{j-1}\sqrt{2\ln(1+m_j)}}{\eps}$, we have
\begin{equation}\label{eqn:Zj-ub}
\E Z_j\leq 2^{1-j}\eps\sqrt{2\ln(1+m_j)}.
\end{equation}
As the total variance of all variables is upper bounded by 1, we know that
\begin{equation}\label{eqn:crux}
1\geq \sum_{Y_i\in G_j}\Sigma_{ii}>\sum_{Y_i\in G_j}2^{-2j}\eps^2=2^{-2j}\eps^2 m_j.
\end{equation}
Then we have $m_j<2^{2j}\eps^{-2}$, and applying to \eqref{eqn:Zj-ub} we get 
\begin{equation}\label{eqn:Zj-ub-final}
\E Z_j\leq 2^{1-j}\eps\sqrt{2\ln(1+2^{2j}\eps^{-2})}<2^{1-j}\eps\sqrt{2+2\ln (2^{2j}\eps^{-2})}=2^{1-j}\eps\cdot \left(\sqrt{j}+\sqrt{\ln\frac{1}{\eps}}\right)\cdot O(1).
\end{equation}
Here the second inequality comes from $\ln(1+x)<1+\ln x$ for $x>1$; the last equality comes from extracting the dominant terms below the square root. By taking the max of $Z_j$ over all groups, we have

$$\E\max(0,\max_{i\in[m]}Y_i)=\E\max_{j}Z_j
\leq \sum_{j} \E Z_j
<\sum_{j}2^{1-j}\eps\left(\sqrt{j}+\sqrt{\ln\frac{1}{\eps}}\right)\cdot O(1)
$$
where the last inequality follows fro \eqref{eqn:Zj-ub-final}. Finally, notice that the sum with $\sqrt{j}$ terms is $O(\eps)$, which has a lower order compared to the sum with $\sqrt{\ln\frac{1}{\eps}}$ terms, 
 from which we conclude that:
 
$$\E\max(0,\max_{i\in[m]}Y_i) \leq O\left(\eps\sqrt{\ln\frac{1}{\eps}}\right)$$
\end{proof}

Secondly, for any two vectors of independent normal variables with small $\ell_1$ distance in variances, their respective expected maximum values are close. Thus to search for a near-optimal solution, it suffices to do a grid search on all possible variance vectors with each positive element being an integral multiple of $poly(\eps)$ (we use $\eps^{-3}$ in Algorithm~\ref{alg:any-mean-ptas}). Any variance vector not being a grid point has a similar objective due to Lemma~\ref{lem:lipschitz}.


\begin{restatable}{lemma}{lemlipschitz}\label{lem:lipschitz}
Consider two vectors of standard deviations $(\sigma_1,\sigma_2,\cdots,\sigma_n)$ and $(\sigma'_1,\sigma'_2,\cdots,\sigma'_n)$. Let $X_1,\cdots,X_n$ and $Y_1,\cdots,Y_n$ be normal variables with $X_i\sim \N(\mu_i,\sigma_i^2)$ and $Y_i\sim \N(\mu_i,\sigma_i'^{2})$, then
\begin{equation*}
|\E\max_{i\in[n]} X_i - \E\max_{i\in[n]} Y_i| \leq O(1)\cdot\sum_{i\in[n]}|\sigma_i-\sigma'_i|.
\end{equation*}
    
\end{restatable}


\begin{algorithm}[htb]
\caption{A PTAS algorithm distributing variances to independent arbitrary-mean variables}
\label{alg:any-mean-ptas}
\begin{algorithmic}[1]
\STATE {\bfseries Input:} $\mu_1,\mu_2,\cdots,\mu_n \geq 0$
\STATE $k \gets \frac{1}{\eps^2}$;
\STATE Set current maximum $M=0$;
\FOR{all possible $k$ indices $1\leq i_1\leq i_2\leq\cdots\leq i_k\leq n$}
\FOR{any $i\not\in \{i_1,\cdots,i_k\}$}
\STATE Set $\hatsigma_i=0$;
\ENDFOR
\FOR{all possible $(\sigma_{i_1},\cdots,\sigma_{i_k})$, where $\sigma_{i_j}$ is an integral multiple of $\eps^3$ for every $i_j\in[n]$}
  \IF{$\E\max_{i\in[n]}X_{i} > M$}
  \STATE $(\hatsigma_{i_1},\cdots,\hatsigma_{i_k})\gets(\sigma_{i_1},\cdots,\sigma_{i_k})$;
  \STATE $M\gets \obj(\boldsigma);$
  \ENDIF
\ENDFOR
\ENDFOR
\STATE {\bfseries Output:} $\hatboldsigma=(\hatsigma_1,\cdots,\hatsigma_n)$
    
\end{algorithmic}
\end{algorithm}

\subsection{Expectation maximization for a single set with correlated variables}

When all random variables come from a multi-dimensional Gaussian distribution, we need to have a different algorithm. While Lemma~\ref{lem:eps-contribution} was already shown for correlated Gaussian variables, which means that we only need to search for covariance matrices with $\text{poly}(\frac{1}{\eps})$ positive dimensions, the smoothness condition from Lemma~\ref{lem:lipschitz} no longer works. Instead, we observe that for any two vectors of random variables with small Earth Mover's distance, their respective expected maximum values are close. Another observation is that for two multi-dimensional Gaussian random vectors, if the square roots of their covariance matrices are close in entry-wise $\ell_1$ distance, then they have a small Earth Mover's distance. 

\begin{restatable}{lemma}{lememdmax}
\label{lem:emd-max}
For two random vectors $X=(X_1,X_2,\cdots,X_n)$ and $Y=(Y_1,Y_2,\cdots,Y_n)$ drawn from distribution $D_X$ and $D_Y$ respectively, if the Earth Mover's Distance $d_{EMD}(D_X,D_Y)$ (defined by $\ell_1$ norm) between two distributions is at most $\eps$, then $|\E\max_i X_i-\E\max_i Y_i|\leq\eps$.
\end{restatable}
\begin{restatable}{lemma}{lememd}\label{lem:emd-w2}
For any two dimension-$k$ Gaussian distributions with $D_1=\N(\mu,A)$ and $D_2=\N(\mu,B)$, suppose that covariance matrices $A$ and $B$ are close in entry-wise $\ell_1$ distance: $\|A-B\|_{1,1}=\sum_{i,j}|A_{ij}-B_{ij}|\leq\eps$. Then the Earth Mover Distance between $D_1$ and $D_2$ is bounded by $O(k^{2.75}\eps^{0.5})$.
\end{restatable}


Combining the previous observations, to search for a near-optimal solution, it suffices to do a grid search on the covariance matrices with each element being an integral multiple of $\poly(\eps)$. We can design the following efficient algorithm that obtains a near-optimal solution with $O(\eps)$ loss for any fixed $\eps>0$. The algorithm performs grid search on $O(\eps^{-2})$ of the variables with each variable having variance being multiples of $\poly(\eps)$ (we use $\eps^{8.5}$ here due to a worse smoothness lemma).
\begin{algorithm}[htb]
\caption{A PTAS algorithm distributing variances to arbitrary-mean correlated variables}
\label{alg:any-mean-correlated-ptas}
\begin{algorithmic}
\STATE {\bfseries Input:} $\mu_1,\mu_2,\cdots,\mu_n \geq 0$
\STATE $k \gets \frac{1}{\eps^2}$;
\STATE Set current maximum $M=0$;
\FOR{all possible $k$ indices $1\leq i_1\leq i_2\leq\cdots\leq i_k\leq n$}
\FOR{any $i\not\in \{i_1,\cdots,i_k\}$}
\FOR{any $\ell\in[n]$}
\STATE set $\widehat{\Sigma}_{i\ell}=\widehat{\Sigma}_{\ell i}=0$;
\ENDFOR
\ENDFOR
\FOR{any $i,j\in \{i_1,\cdots,i_k\}$}
\FOR{all possible values of $(\Sigma_{ij})$ being an integral multiple of $\eps^3$ with value in [-1,1]}
  \IF{ $\Sigma$ is positive semi-definite and $\obj(\Sigma) > M$}
  \STATE $\widehat{\Sigma}\gets\Sigma$;
  \STATE $M\gets \obj(\Sigma)$;
  \ENDIF
\ENDFOR
\ENDFOR
\ENDFOR
\STATE {\bfseries Output:} $\widehat{\Sigma}$
\end{algorithmic}
\end{algorithm}

\subsection{Expectation maximization for multiple sets}

When there are multiple sets, the problem is much more complicated. The PTAS technique for the single-set setting no longer works as the scale of the optimal objective can be very different. For example, when each set have small cardinality, the optimal way to distribute the variance may become uniformly split the variance budget to all variables, which means that all variables have $O(\frac{1}{n^2})$ variance. The previous grid search algorithm will no longer be efficient as there can be $\Omega(n)$ variables with small variance and contribute significantly to the optimal objective. One example is when each set contains two variables and the sets form a cycle graph.

\begin{restatable}{theorem}{thmcycle}\label{thm:cycle}
Suppose that $m=n$, and $S_j=\{X_j,X_{j+1}\}$ for every $j\in[n]$ (assume that $X_{n+1}=X_1$). When all variables have the same mean $\mu$, the optimal way to distribute variance is to set $\sigma_i^2=\frac{1}{n}$ for every $i$.
\end{restatable}

For the most general case of the problem, we obtain an efficient approximation algorithm with a competitive ratio $O(\log n)$. The main intuition is as follows. Firstly, for any solution, variables with tiny variance $<\frac{1}{n}$ only have negligible contribution to the total objective. For the rest of the variables, we can group the variables according to their variances such that each group contains variables with variance differences bounded by a constant factor. Then since there are only $O(\log n)$ groups, one group of variables contribute $\Omega(\frac{1}{\log n})$ fraction of the optimal objective. Secondly, when the standard deviation of each variable perturbs by a factor of 2, the overall objective also only changes by at most a factor of 2. This means that if we can solve the problem with each variable having a variance of either 0 or some fixed value, we can obtain $\Omega(\frac{1}{\log n})$ fraction of the optimal objective.

\begin{restatable}{lemma}{varapprox}\label{lem:var2approx}
For $n$ zero-mean normally distributed variables $X_1,X_2,\cdots,X_n$ with standard deviation $\sigma_1,\cdots,\sigma_n$, suppose that there are $n$ other zero-mean normally distributed variables $X'_1,X'_2,\cdots,X'_n$ with standard deviation $\sigma'_1,\cdots,\sigma'_n$ satisfying $\sigma_i\leq \sigma'_i\leq 2\sigma_i$, then $\E\max X'_i\leq \E\max X'_i\leq 2\E\max X_i$.
\end{restatable}

The last observation is that, for each set, the marginal gain of setting one more variable with a fixed variance is a decreasing function. 

\begin{restatable}{lemma}{submodularset}\label{lem:submodularset}
For $n$ zero-mean normally distributed variables $X_1,X_2,\cdots,X_n$ and $\sigma_0>0$, let $f:2^{[n]}\to \mathbb{R}$ be a set function such that for any $S\subseteq [n]$, $f(S)=\E\max_{i\in[n]}X_i$ where $X_i\sim \N(0,\sigma_0^2)$ for $i\in S$ and $X_i\equiv0$ for $i\not\in S$. Then $f$ is a submodular function.
\end{restatable}

This means that the problem with each variable having a variance being either 0 or some fixed value is actually a submodular maximization problem, and can be solved efficiently with approximation factor $1-\frac{1}{e}$ via a greedy algorithm. The complete approximation algorithm for multiple sets is summarized in Algorithm~\ref{alg:general-log-approx}.

\makeatletter
\setlength{\@fptop}{0pt}
\makeatother
\begin{algorithm}[t]
\caption{An $O(\log n)$-approx algorithm distributing variances to variables in multiple sets}
\label{alg:general-log-approx}
\begin{algorithmic}
    \STATE {\bfseries Input:} Means $(\mu_1,\cdots,\mu_n)$, Sets $S_1,\cdots,S_m$
    \STATE Initialize current best objective $M \gets 0$;
    \STATE Initialize current best $(\hatsigma_1,\cdots,\hatsigma_n)\gets(0,\cdots,0)$;
    \STATE Remove all sets with $|S_j|=1$;
    \FOR{$k\gets 0\ \textrm{ to }\log_2 n$}
        \STATE $(\sigma_1,\cdots,\sigma_n)\gets (0,\cdots,0)$;
    \WHILE{$\#$ variables with variance $4^{-k}$ $<\min(4^{k},n)$}
    \STATE Find the variable $i$ with $\sigma_i=0$ such that after setting $\sigma_i=4^{-k}$, $\obj(\boldsigma)$ is maximized;
    \STATE $\sigma_i\gets4^{-k}$;
    \ENDWHILE
  \IF{$\obj(\boldsigma)>M$}
    \STATE  $M\gets \obj(\boldsigma)$;
    \STATE  $\hatboldsigma \gets \boldsigma$;
    \ENDIF
    \ENDFOR
   \STATE {\bfseries Output:} $(\hatsigma_1,\cdots,\hatsigma_n)$
\end{algorithmic}

\end{algorithm}

For \genbcorr\ that allows covariance, exactly the same algorithm gives a logarithmic approximation since the optimal objective between the independent case \genb\ and the correlated case \genbcorr\ are bounded by a constant due to a constant correlation gap of the optimization problem.

\begin{restatable}{lemma}{lemcorrindgap}
\label{lem:corr-ind-gap}
For any multi-dimensional Gaussian vector $(X_1,\cdots,X_n)\sim\N(\mu,\Sigma)$, consider independent Gaussian variables with the same means and variances $(Y_1,\cdots,Y_n)\sim \N(\mu,diag(\Sigma))$. Then
$$\E\max(X_1,\cdots,X_n)\leq \frac{2e}{e-1}\E\max(Y_1,Y_2,\cdots,Y_n).$$
\end{restatable}


\clearpage
\bibliographystyle{abbrvnat}
\bibliography{main}
\clearpage
\appendix

\section{More Details on Variance Allocation in One Set}
\label{sec:singleset}

\subsection{Variance allocation for multiple independent variables} \label{subsec:independent}

In this section, we propose several efficient approximation schemes that find near-optimal ways to distribute the variance across all random variables. We start with the simplest case \basic, where there are $n$ normally distributed random variables $X_1\sim \N(0,\sigma_1^2),X_2\sim \N(0,\sigma_2^2),\cdots,X_n\sim \N(0,\sigma_n^2)$, and the to be distributed variances satisfy $\sigma_1^2+\sigma_2^2+\cdots+\sigma_n^2=1$. 

Recall that $\opt$ denotes the optimal objective of the following optimization problem:
\begin{eqnarray*}
\opt&=&\max_{\sigma_1,\cdots,\sigma_n}\E\max_{i\in [n]} X_i\\
\textrm{subject to}& &\sum_{i=1}^{n}\sigma_i^2= 1.
\end{eqnarray*}

We have the following theorem.
\begin{theorem}\label{thm:PTAS-zeromean}
For any constant $\eps>0$, there exists an algorithm with running time polynomial in $n$, which computes a variance vector $\boldsigma=(\sigma_1,\cdots,\sigma_n)$ satisfying the optimization constraint, such that for $X_1\sim \N(0,\sigma_1^2),X_2\sim \N(0,\sigma_2^2),\cdots,X_n\sim \N(0,\sigma_n^2)$, $\E\max_{i\in [n]} X_i\geq \opt-\eps$.
\end{theorem}

The main intuition for the PTAS is as follows. Firstly, all random variables with variance less than $\eps^2$ can be ignored, as their total contribution to the objective is negligible (i.e. $\tilde{O}(\eps)$). Secondly, the objective is Lipschitz with respect to each $\sigma_i$, thus we may only need to consider solutions with discrete $\sigma_i$ values. The above intuitions are formally described by Lemma~\ref{lem:eps-contribution} (main paper) and Lemma~\ref{lem:lipschitz} (below).

\lemlipschitz*
To prove Lemma~\ref{lem:lipschitz}, we prove the following much stronger lemma, and Lemma~\ref{lem:lipschitz} can be viewed as a corollary by repeatedly applying the Lemma~\ref{lem:lipschitz-strong} to each index $i\in[n]$.

\begin{lemma}\label{lem:lipschitz-strong}
Consider two normal variables $X\sim \N(\mu,\sigma_1)$ and $Y\sim\N(\mu,\sigma_2)$ with the same mean. For any random variable $T$ that is independent with $X$ and $Y$, let $X'=\max(X,T)$, and $Y'=\max(Y,T)$. Then
\begin{equation*}
|\E X'-\E Y'| \leq O(1)\cdot |\sigma_1-\sigma_2|.
\end{equation*}
\end{lemma}


\begin{proof}[Proof of Lemma~\ref{lem:lipschitz-strong}]

Firstly, we can assume that $X$ and $Y$ have mean 0 by shifting $X,Y,T$ by $\mu$ (thus $X'$ and $Y'$ are also shifted by $\mu$). Secondly, observe that we only need to prove the lemma for $T=t\in\R$ being a constant, and in this case, $X'=\max(X,t)$, and $Y'=\max(Y,t)$. With the lemma proven for a constant $t$, the original lemma can be shown by taking the expectation over all $t\sim T$. 

Without loss of generality, assume that $\sigma_1 > \sigma_2$. Let $\Phi_1$ and $\Phi_2$ be the CDF of $X$ and $Y$, and $\phi_1$ and $\phi_2$ be the PDF $X$ and $Y$ respectively. Suppose we define $\phi(x)=\frac{1}{\sqrt{2\pi}}e^{-x^2/2}$ to be the PDF of a standard normal distribution with mean 0 and variance 1, and $\Phi(x)=\int_{0}^{x}\phi(z)dz$ to be the CDF of a standard normal distribution, then $\phi_i(x)=\frac{1}{\sigma_i}\phi(\frac{x}{\sigma_i})$, $\Phi_i(x)=\phi(\frac{x}{\sigma_i})$ for $i=1,2$. Then
\begin{eqnarray*}
|\E X'-\E Y'| & = &\left|\left(\int_{-\infty}^{t}t\phi_1(x)dx+\int_{t}^{\infty}x\phi_1(x)dx\right)-\left(\int_{-\infty}^{t}t\phi_2(x)dx+\int_{t}^{\infty}x\phi_2(x)dx\right)\right|\\
&=& \left|\left(t\Phi_1(t)+\sigma_1\phi\left(\frac{t}{\sigma_1}\right)\right)-\left(t\Phi_2(t)+\sigma_2\phi\left(\frac{t}{\sigma_2}\right)\right)\right|\\
&=&\left|t\left(\Phi\left(\frac{t}{\sigma_1}\right)-\Phi\left(\frac{t}{\sigma_2}\right)\right)+\left(\sigma_1\phi\left(\frac{t}{\sigma_1}\right)-\sigma_2\phi\left(\frac{t}{\sigma_2}\right)\right)\right|\\
&\leq&\left|t\left(\Phi\left(\frac{t}{\sigma_1}\right)-\Phi\left(\frac{t}{\sigma_2}\right)\right)\right|+\left|\sigma_1\phi\left(\frac{t}{\sigma_1}\right)-\sigma_2\phi\left(\frac{t}{\sigma_2}\right)\right|.
\end{eqnarray*}
Here the first equality is from the definition of $X'$ and $Y'$; the second equality is from the expectation of truncated normal distributions. Let
$$Z_1=t\left(\Phi\left(\frac{t}{\sigma_1}\right)-\Phi\left(\frac{t}{\sigma_2}\right)\right)$$ and
$$Z_2=\sigma_1\phi\left(\frac{t}{\sigma_1}\right)-\sigma_2\phi\left(\frac{t}{\sigma_2}\right).$$ 
To prove the lemma, it suffices to show that both $|Z_1|$ and $|Z_2|$ are bounded by $O(|\sigma_1-\sigma_2|)$. As replacing $t$ with $-t$ does not change the value of both terms, we can assume without loss of generality that $t\geq 0$.
\paragraph{Bounding} $|Z_1|$. We prove that when $\sigma_1>2\sigma_2$, $|Z_1|$ is upper bounded by $O(\sigma_1)$; when $\sigma_2\leq \sigma_1\leq 2\sigma_2$, $|Z_1|$ is upper bounded by $O(|\sigma_1-\sigma_2|)$.

If $\sigma_1>2\sigma_2$, 
\begin{eqnarray*}
|Z_1|&=&t\cdot\frac{1}{\sqrt{2\pi}}\int_{t/\sigma_1}^{t/\sigma_2}e^{-x^2/2}dx\\
&\leq&\frac{t}{\sqrt{2\pi}}\int_{t/\sigma_1}^{\infty}e^{-x^2/2}dx\\
&=&\frac{t}{\sqrt{2\pi}}\int_{0}^{\infty}e^{-(x+\frac{t}{\sigma_1})^2/2}dx\\
&\leq&\frac{t}{\sqrt{2\pi}}\int_{0}^{\infty}e^{-(x^2+(\frac{t}{\sigma_1})^2)/2} dx\\
&=&\frac{1}{2}te^{-\frac{t^2}{2\sigma_1^2}}\leq \frac{1}{2\sqrt{e}}\sigma_1\leq \frac{1}{\sqrt{e}}(\sigma_1-\sigma_2).
\end{eqnarray*}
Here the first line is from the definition of $\Phi$; the last line is by $\frac{1}{\sqrt{2\pi}}\int_{0}^{\infty}e^{-x^2/2}dx=\frac{1}{2}$, and function $\frac{1}{2}xe^{-x^2/2}$ is maximized at $x=1$ with maximum value $\frac{1}{2\sqrt{e}}$.

If $\sigma_1\leq 2\sigma_2$, 
\begin{eqnarray*}
|Z_1|&=&t\cdot\frac{1}{\sqrt{2\pi}}\int_{t/\sigma_1}^{t/\sigma_2}e^{-x^2/2}dx\\
&\leq&\frac{t}{\sqrt{2\pi}}\left(\frac{t}{\sigma_2}-\frac{t}{\sigma_1}\right)e^{-\frac{t^2}{2\sigma_1^2}}\\
&=&\frac{1}{\sqrt{2\pi}}\cdot\frac{t^2}{\sigma_1\sigma_2}e^{-\frac{t^2}{2\sigma_1^2}}(\sigma_1-\sigma_2)\\
&\leq&\sqrt{\frac{2}{\pi}}\cdot \frac{t^2}{\sigma_1^2}e^{-\frac{t^2}{2\sigma_1^2}}(\sigma_1-\sigma_2)\\
&\leq&\sqrt{\frac{2}{\pi}}\cdot\frac{2}{e}(\sigma_1-\sigma_2).
\end{eqnarray*}
Here the second line is by $e^{-x^2/2}$ is decreasing on $[\frac{t}{\sigma_1},\frac{t}{\sigma_2}]$; the fourth line is by $\sigma_1\leq 2\sigma_2$; the last line is by function $xe^{-x/2}$ is maximized at $x=2$ with maximum value $\frac{2}{e}$. Thus in both cases, we have shown that $|Z_1|$ is upper bounded by $O(|\sigma_1-\sigma_2|)$ (with small constant $<1$).

\paragraph{Bounding} $|Z_2|$. To show that $|Z_2|=\left|\sigma_1\phi\left(\frac{t}{\sigma_1}\right)-\sigma_2\phi\left(\frac{t}{\sigma_2}\right)\right|$ is bounded by $O(|\sigma_1-\sigma_2|)$, it suffices to show that function $g(x)=x\phi\left(\frac{t}{x}\right)$ has bounded derivative. In fact, we can rewrite $g'(x)$ as
\begin{equation*}
g'(x)=\phi\left(\frac{t}{x}\right)+x\phi'\left(\frac{t}{x}\right)= \phi\left(\frac{t}{x}\right)+\frac{1}{\sqrt{2\pi}}\cdot\frac{t^2}{x^2}e^{-\frac{t^2}{2x^2}}
\end{equation*}
being the sum of two positive terms with each term being bounded by some absolute constant. Actually when $x=t$, $|g'(x)|$ is maximized with value $\sqrt{\frac{2}{e\pi}}$. Thus by the mean value theorem, $|Z_2|\leq \sqrt{\frac{2}{e\pi}}(\sigma_1-\sigma_2)$.

Combining the bounds for $|Z_1|$ and $|Z_2|$, we get $|\E X'-\E Y'|\leq O(|\sigma_1-\sigma_2|)$, finishing the proof of the theorem.

\end{proof}

\noindent Now with the help of Lemma~\ref{lem:eps-contribution} and Lemma~\ref{lem:lipschitz}, we are ready to prove Theorem~\ref{thm:PTAS-zeromean}.

\begin{proof}[Proof of Theorem~\ref{thm:PTAS-zeromean}]
Let k = $\frac{1}{\eps^2}$. We first show that we only need to distribute variance to at most $k$ variables.

When $n\geq k$, first observe that $\E\max_{i\in [n]} X_i$ and $\E\max(0,\max_{i\in [n]} X_i)$ are close (up to $O(2^{-1/\eps^2})$). This is true since $\Pr[\max_{i\in [n]} X_i < 0] = 2^{-k}$, and $\E[\max_{i\in [n]} X_i|\max_{i\in [n]} X_i < 0] \geq \E[X_1|X_1<0]=\sqrt{\frac{2}{\pi}}\sigma_1=O(1)$ for $\sigma_1\leq 1$. 

Then for any variance allocation $\sigma_1\geq\sigma_2 \geq\cdots\geq\sigma_n$, we know that the objective value 
\[\E\max(0,\max_{i\in [n]} X_i)\leq \E\max(0,\max_{i\in [k]} X_i)+\E\max(0,\max_{k+1\leq i\leq n} X_i)=\E\max(0,\max_{i\in [k]} X_i)+O(\eps)\]
by Lemma~\ref{lem:eps-contribution}\footnote{For clearance of presentation we do not optimize the small terms.}. Therefore, an algorithm that wants to obtain $O(\eps)$ loss does not need to distribute variance $<\eps^2$ to any variable. Also, as the total variance budget is 1, there can be at most $\frac{1}{\eps^2}$ variables with variance $>\eps^2$, thus $k\leq \frac{1}{\eps^2}$. This means that an algorithm that wants to obtain $O(\eps)$ loss only needs to distribute at least $\eps^2$ variance to at most $\frac{1}{\eps^2}$ variables.

By Lemma~\ref{lem:lipschitz}, if an algorithm can obtain the optimal objective conditioned on each $\eps_i$ being an integer times $\eps^3$, the algorithm will have a loss at most $k\eps^3$, if at most $k$ variables have positive variance. Thus the following algorithm has $O(\eps)$ loss compared to the optimal objective:

\begin{algorithm}
\caption{A PTAS algorithm distributing variances to zero-mean variables}
\label{alg:zero-mean-ptas}
\begin{algorithmic}
\STATE {\bfseries Input:} $n \geq 0$
\STATE $k \gets \frac{1}{\eps^2}$;
\STATE Set current maximum $M=0$;
\FOR{any $i>k$}
\STATE $\hatsigma_i\gets0$
\ENDFOR
\FOR{all possible $(\sigma_1,\cdots,\sigma_k)$, where $\hatsigma_i$ is an integral multiple of $\eps^3$ for every $i\in[k]$}
  \IF{$\E\max_{i\in [k]}X_i > M$}
  \STATE Set $(\hatsigma_1,\cdots,\hatsigma_k)\gets(\sigma_1,\cdots,\sigma_k)$;
  \STATE $M\gets \obj(\boldsigma)$;
  \ENDIF
\ENDFOR
\STATE {\bfseries Output:} $(\hatsigma_1,\cdots,\hatsigma_n)$
\end{algorithmic}
\end{algorithm}

Now we analyze the running time of the algorithm. There are at most $(\eps^{-3})^{1/\eps^2}$ different vectors of $\hatboldsigma$, and for each vector, we can efficiently compute the max of $k$ normal variables with given variance. Thus the algorithm is a polynomial time approximation scheme of the original objective. 

\end{proof}

\noindent The algorithm can be easily extended to the case where the means of the variables are no longer 0.

\thmptasindependent*

\begin{proof}[Proof of Theorem~\ref{thm:PTAS-anymean}]
For any variance allocation $\sigma_1,\cdots,\sigma_n$, let $G_1\cup G_2$ be the partition of all variables such that $G_1$ contains all variables $X_i$ with $\sigma_i>\eps^2$; $G_2$ contains all variables $X_i$ with $\sigma_i\leq\eps^2$. For each variable $X_i$, define $X'_i=X_i-\mu_i$. Then $X'_i\sim\N(0,\sigma_i^2)$ is a zero-mean variable. We can bound the original objective as follows:
\begin{eqnarray*}
\E\max_{i\in[n]} X_i&=&\E\max\left(\max_{X_i\in G_1}X_i, \max_{X_i\in G_2}X_i\right)\\
&\leq&\E\max\left(\max_{X_i\in G_1}X_i, \max_{X_i\in G_2}\mu_i+\max\left(0,\max_{X_i\in G_2}X'_i\right)\right)\\
&\leq&\E\max\left(\max_{X_i\in G_1}X_i, \max_{X_i\in G_2}\mu_i\right)+\E\max\left(0,\max_{X_i\in G_2}X'_i\right)\\
&=&\E\max\left(\max_{X_i\in G_1}X_i, \max_{X_i\in G_2}\mu_i\right)+O(\eps).
\end{eqnarray*}
Here the first line is by splitting the objective into two groups of variables; the second line is by decomposing $X_i$ to $X_i'+\mu_i$, and upper-bounding $X_i'$ bt $\max(0,X_i')$; the third line is by getting the positive term outside of the outer maximization; the last line is by using Lemma~\ref{lem:eps-contribution} to bound the maximum of multiple zero-mean variables. This means that to get $O(\eps)$ loss in the objective, it suffices to only distribute variance to $O(\frac{1}{\eps^2})$ variables, with each having variance at least $\eps^2$. 

\setcounter{algorithm}{0}
\begin{algorithm}[htb]
\caption{A PTAS algorithm distributing variances to independent arbitrary-mean variables}
\begin{algorithmic}[1]
\STATE {\bfseries Input:} $\mu_1,\mu_2,\cdots,\mu_n \geq 0$
\STATE $k \gets \frac{1}{\eps^2}$;
\STATE Set current maximum $M=0$;
\FOR{all possible $k$ indices $1\leq i_1\leq i_2\leq\cdots\leq i_k\leq n$}
\FOR{any $i\not\in \{i_1,\cdots,i_k\}$}
\STATE Set $\hatsigma_i=0$;
\ENDFOR
\FOR{all possible $(\sigma_{i_1},\cdots,\sigma_{i_k})$, where $\sigma_{i_j}$ is an integral multiple of $\eps^3$ for every $i_j\in[n]$}
  \IF{$\E\max_{i\in[n]}X_{i} > M$}
  \STATE $(\hatsigma_{i_1},\cdots,\hatsigma_{i_k})\gets(\sigma_{i_1},\cdots,\sigma_{i_k})$;
  \STATE $M\gets \obj(\boldsigma);$
  \ENDIF
\ENDFOR
\ENDFOR
\STATE {\bfseries Output:} $\hatboldsigma=(\hatsigma_1,\cdots,\hatsigma_n)$
    
\end{algorithmic}
\end{algorithm}

Similar to Theorem~\ref{thm:PTAS-zeromean}, we can use a similar algorithm to enumerate all possible ways to distribute variances being multiple of $\eps^3$, see Algorithm~\ref{alg:any-mean-ptas} for details. The only difference from the zero-mean Algorithm~\ref{alg:zero-mean-ptas} is that we now need to guess which $\frac{1}{\eps^2}$ variables should be distributed variances. The algorithm runs in polynomial time for any fixed $\eps>0$, and obtains a solution with $O(\eps)$ loss compared to the optimal objective.

\end{proof}

\subsection{Variance distribution for correlated variables}
\label{subsec:correlated}
We now proceed to the more general case, where there are $n$ jointly normally distributed random variables $(X_1,X_2,\cdots,X_n)\sim\N(\mu,\Sigma)$. We still assume that the total variance of the variables is bounded: $\sum_{i=1}^{n}\Sigma_{ii}\leq 1$, and we are also free to determine the positive semi-definite covariance matrix $\Sigma$ satisfying the constraint.  
Let $\opt$ denote the optimal objective in this setting:
\begin{eqnarray*}
\opt&=&\max_{\Sigma}\ \E\max_{i\in [n]} X_i\\
\textrm{subject to}& &\sum_{i=1}^{n}\Sigma_{ii}= 1.
\end{eqnarray*}

We have the following theorem.

\thmptascorrelated*

The algorithm would be almost the same as Algorithm~\ref{alg:any-mean-ptas}. Although we proved that all random variables with small variance $<\eps^2$ can be ignored in Lemma~\ref{lem:eps-contribution} for the correlated case, the Lemma that proves the smoothness of the objective with respect to the covariance is only shown for the independent case (Lemma~\ref{lem:lipschitz}). Therefore, we need to get a stronger smoothness lemma. We prove the following two lemmas: one shows that for any two joint distributions over $n$ variables, if they have small Earth Mover's Distance, then their objective $\obj$ would be close; the other shows that for two Gaussian distributions with covariance matrices close element-wise, their Earth Mover's distance is small.

\lememdmax*

\begin{proof}[Proof of Lemma~\ref{lem:emd-max}]
Consider the coupling $\gamma$ between $D_X$ and $D_Y$ that defines the Earth Mover's Distance 
$$d_{EMD}(D_X,D_Y)=\E_{(X,Y)\sim \gamma}\|X-Y\|_1.$$
Notice that for any two vectors $X$ and $Y$, $|\max_i X_i-\max_i Y_i|\leq \|X-Y\|_1$. Thus
$$|\E\max_i X_i-\E\max_i Y_i|=\E_{(X,Y)\sim \gamma}|\E\max_i X_i-\E\max_i Y_i|\leq \E_{(X,Y)\sim \gamma}\|X-Y\|_1=d_{EMD}(D_X,D_Y)\leq\eps.$$
\end{proof}

\lememd*

\begin{proof}[Proof of Lemma~\ref{lem:emd-w2}]
For any zero-mean multi-dimensional Gaussian distributions with Positive Semi-Definite (PSD) matrices $A$ and $B$, we upper bound the Earth Mover's Distance with their Wasserstein-2 ($W_2$) distance. The $W_2$ distance is defined similarly to the EMD distance, where EMD distance describes the smallest effort to transport between the two distributions using $\ell_1$ distance, while $W_2$ distance uses $\ell_2$ distance in the coupling. It is known that 
$$d_{EMD}(D_1,D_2)\leq\sqrt{k}\cdot d_{W_2}(D_1,D_2),$$
and
$$d_{W_2}(D_1,D_2)=\tr\left(A+B-2(A^{1/2}BA^{1/2})^{1/2}\right).$$

Now we look at the entry-wise $\ell_1$ distance between $A^2$ and $A^{1/2}BA^{1/2}$. As $\tr(A),\tr(B)\leq1$, we have every element in $A,B,A^{1/2},B^{1/2}$ have absolute value bounded below 1.
\begin{eqnarray*}
\|A^{1/2}A-A^{1/2}B\|_{1,1}&=&\sum_{i,j}\left|\sum_{r}A^{1/2}_{ir}(A-B)_{rj}\right|\\
&\leq&\sum_{r}\left(\sum_{i}|A^{1/2}_{ir}|\right)\left(\sum_{j}(A-B)_{rj}\right)\\
&\leq&\sum_{r}\left(\sum_{i}1\right)\left(\sum_{j}(A-B)_{rj}\right)\\
&\leq&k\|A-B\|_{1,1}\leq k\eps.
\end{eqnarray*}
Similarly 
$$\|A^{1/2}AA^{1/2}-A^{1/2}BA^{1/2}\|_{1,1}\leq k\|A^{1/2}A-A^{1/2}B\|_{1,1}\leq k^2\eps.$$

Notice that for two matrices with small $\ell_1$ distance, their spectrum is similar. Let $P=A^2$, $Q=A^{1/2}BA^{1/2}$. Let $\lambda_i(P)$ denote the $i$th largest eigenvalue of $P$. By Weyl's inequality,
\begin{eqnarray*}
|\lambda_i(P)-\lambda_i(Q)|\leq \|P-Q\|_{op}\leq \sqrt{k}\|P-Q\|_{1,1}\leq k^{2.5}\eps.
\end{eqnarray*}
Here $\|\cdot\|_{op}$ is the operator norm. Then 
\begin{eqnarray*}
|\tr(P^{1/2})-\tr(Q^{1/2})|&=&\left|\sum_{i=1}^{k}\sqrt{\lambda_i(P)}-\sum_{i=1}^{k}\sqrt{\lambda_i(Q)}\right|\\
&\leq&\sum_{i=1}^{k}\left|\sqrt{\lambda_i(P)}-\sqrt{\lambda_i(Q)}\right|\\
&\leq&\sum_{i=1}^{k}\left|\sqrt{|\lambda_i(P)-\lambda_i(Q)|}\right|\\
&\leq&k\cdot k^{1.25}\eps^{0.5}=k^{2.25}\eps^{0.5}.
\end{eqnarray*}

Thus
\begin{eqnarray*}
d_{W_2}(D_1,D_2)&=&\tr(A)+\tr(B)-2\tr(Q^{1/2})\\
&=&\tr(A)+\tr(B)-2\tr(Q^{1/2})+2\tr(P^{1/2})-2\tr(A)\\
&\leq&|\tr(A)-\tr(B)|+2|\tr(P^{1/2})-\tr(Q^{1/2})|\\
,
\end{eqnarray*}
and $d_{EMD}(D_1,D_2)\leq \sqrt{k}d_{W_2}(D_1,D_2)=O(k^{2.75}\eps^{0.5})$
\end{proof}

\noindent Now we are ready to propose and analyze the algorithm for Theorem~\ref{thm:PTAS-correlated}.

\begin{proof}[Proof of Theorem~\ref{thm:PTAS-correlated}]

The algorithm and its analysis are almost identical to Theorem~\ref{thm:PTAS-anymean}, except we need a more fine-grained search due to the loss in smoothness lemma (Lemma~\ref{lem:emd-w2} vs. Lemma~\ref{lem:lipschitz}). Let $k=\frac{1}{\eps^2}$ and $\delta$ be to-be-determined parameters. When we search for the variance allocation, instead of searching for the variance vector with at most $k$ positive terms, we search for the covariance matrix of size at most $k^2$ for a subset of variables. As $|\Sigma_{ij}|\leq\sqrt{\Sigma_{ii}\Sigma_{jj}}\leq 1$, we only need to search for a covariance matrix with each element being integral multiples of $\delta$ bounded between $-1$ and $1$.

The $\ell_1$ accuracy of the covariance matrix we get is $k^2\delta$. By Lemma~\ref{lem:emd-w2}, the Earth Mover's distance of the optimal solution and the matrix we get from the grid search algorithm is $O(k^{2.75}\cdot\sqrt{k^2\delta})=O(k^{3.75}\delta)$. Thus a grid size $\delta=\eps^{8.5}$ is enough. See Algorithm~\ref{alg:any-mean-correlated-ptas} for the full algorithm.

\begin{algorithm}[htb]
\caption{A PTAS algorithm distributing variances to arbitrary-mean correlated variables}
\begin{algorithmic}
\STATE {\bfseries Input:} $\mu_1,\mu_2,\cdots,\mu_n \geq 0$
\STATE $k \gets \frac{1}{\eps^2}$;
\STATE Set current maximum $M=0$;
\FOR{all possible $k$ indices $1\leq i_1\leq i_2\leq\cdots\leq i_k\leq n$}
\FOR{any $i\not\in \{i_1,\cdots,i_k\}$}
\FOR{any $\ell\in[n]$}
\STATE set $\widehat{\Sigma}_{i\ell}=\widehat{\Sigma}_{\ell i}=0$;
\ENDFOR
\ENDFOR
\FOR{any $i,j\in \{i_1,\cdots,i_k\}$}
\FOR{all possible values of $(\Sigma_{ij})$ being an integral multiple of $\eps^3$ with value in [-1,1]}
  \IF{ $\Sigma$ is positive semi-definite and $\obj(\Sigma) > M$}
  \STATE $\widehat{\Sigma}\gets\Sigma$;
  \STATE $M\gets \obj(\Sigma)$;
  \ENDIF
\ENDFOR
\ENDFOR
\ENDFOR
\STATE {\bfseries Output:} $\widehat{\Sigma}$
\end{algorithmic}
\end{algorithm}

\end{proof}

\section{More Details on Variance Allocation in Multiple Subsets}
\label{sec:multipleset}

In this section, we study the multi-subset setting where there are $m>1$ sets. Recall that each set's contribution is again defined by $\max_{i\in S_j}X_i$, and our objective is to maximize the total expected contribution of the $m$ sets. We assume that the $n$ variables $X_i\sim \N(\mu_i,\sigma_i^2)$ are independent and normally distributed, and we need to distribute a total variance budget $1$. To be more precise, the optimal objective $\opt$ is defined as

\begin{eqnarray*}
\opt&=&\max_{\sigma_1,\cdots,\sigma_n}\E\sum_{j=1}^{m}\max_{i\in S_j} X_i\\
\textrm{subject to}& &\sum_{i=1}^{n}\sigma_i^2= 1.
\end{eqnarray*}

An equivalent way to define the instance is to use a bipartite graph. The characterization graph of the instance is a bipartite graph with $n$ vertices on one side representing variables $X_i$ for $i\in[n]$, and $m$ vertices on the other side representing sets $S_j$ for $j\in[n]$, with edge $(i,j)$ exists if and only if $i\in S_j$.

When the sets are large in size, the PTAS algorithms in the previous sections are still applicable. Intuitively, this is true since $\opt$ is of order $\Theta(m)$, and an additive loss $O(\eps)$ for each set leads to a total loss $O(m\eps)$. However, when the sets are relatively small, it is possible that the optimal variance allocation algorithm splits the total budget evenly to all variables, resulting in each set having $o(\eps)$ contribution to the total objective. In this case, a grid search algorithm with grid points being small constants is even not sufficient to get a subpolynomial competitive ratio.

\thmcycle*

\begin{proof}[Proof of Theorem~\ref{thm:cycle}]
For the special case where there are only two variables in each set, the optimization objective can be expressed by $\sigma$ in a closed form. Notice that the contribution of each set $S_j$ is
$$\E\max(X_{j},X_{j+1})=\E\frac{X_{j}+X_{j+1}+|X_{j}-X_{j+1}|}{2}=\mu+\frac{1}{2}\E|X_{j}-X_{j+1}|.$$
As $X_{j}\sim \N(\mu,\sigma_j^2)$ and $X_{j+1}\sim \N(\mu,\sigma_{j+1}^2)$, $|X_{j}-X_{j+1}|$ is the absolute value of a normally distributed variable with variance $\sigma_{j}^2+\sigma_{j+1}^2$, and has mean $\sqrt{\frac{2}{\pi}}\cdot\sqrt{\sigma_{j}^2+\sigma_{j+1}^2}$. Thus, the total objective can be written as
\begin{eqnarray*}
\sum_{j=1}^{n}\E\max(X_{j},X_{j+1})&=&\sum_{j=1}^{n}\left(\mu+\sqrt{\frac{1}{2\pi}}\cdot\sqrt{\sigma_{j}^2+\sigma_{j+1}^2}\right)\\
&=&n\mu+\sqrt{\frac{1}{2\pi}}\sum_{j=1}^{n}\sqrt{\sigma_{j}^2+\sigma_{j+1}^2}.
\end{eqnarray*}
Thus it suffices to optimize the second term. By Cauchy-Schwartz inequality, 
\begin{eqnarray*}
\left(\sum_{j=1}^{n}\sqrt{\sigma_{j}^2+\sigma_{j+1}^2}\right)^2&\leq&\left(\sum_{j=1}^{n}\left(\sigma_{j}^2+\sigma_{j+1}^2\right)\right)\left(\sum_{j=1}^{n}1^2\right)\\
&=& 2n.
\end{eqnarray*}
The equality holds if and only if $\sigma_j+\sigma_{j+1}=\sigma_{j+1}+\sigma_{j+2}=\frac{2}{n}$ for every $j\in[n]$. Thus, setting all variables to have variance $\frac{1}{n}$ is an optimal solution to the original optimization problem.

\end{proof}

\subsection{Instance generated by a random graph}
In this section, we study the setting where all variables are independent with mean $\mu_i=0$ and total variance budget 1.
The bipartite characterization graph generation follows Erd\"{o}s–Rényi model with each edge appearing in the graph with probability $p$. In other words, for any variable $X_i$, $i\in S_j$ with probability $p$.

We observe that for a denser random graph (with larger $p$), fewer variables will have non-negligible variance in the optimal allocation. This is formally characterized in the following theorem.

\thmrandomgraphordering*

\begin{proof}[Proof of Theorem~\ref{thm:random-graph-ordering}]
We first give an upper bound on the objective we can obtain in the random graph. Consider the following variance allocation: $\sigma^2_1=\sigma^2_2=\cdots=\sigma^2_{\frac{1}{p}}=p$, $\sigma_{\frac{1}{p}+1}=\cdots=\sigma_n=0$. Then for every set $S_j$, as any $i\in[n]$ appears in $S_j$ with probability $p$, the probability that $S_j$ contains at least one element in range $[1,\frac{1}{p}]$ is at least $1-(1-\frac{1}{p})^{p}>1-\frac{1}{e}$. Notice that when at least one $X_i$ in $S_j$ has variance $p$, the contribution of $S_j$ to the objective is at least $\E\max(X_1,0)=\sqrt{\frac{2p}{\pi}}=\Omega(\sqrt{p})$. Thus, the optimal objective of the random graph is at least $\Omega(m\sqrt{p})$.

As the total variance budget of all variables is $1$, there can be at most $O(\frac{1}{p})$ variables with large variance $\Omega(p)$. Now we show that if there are $o(\frac{1}{p})$ variables with variance $\Omega(p)$, then the total objective value is much smaller than the objective obtained by the allocation in the previous paragraph. To be more precise, we prove that for small constant $\eps>0$, if there are less than $\frac{\eps}{p}$ variables with variance at least $\eps p$, the total objective is only $o(m\sqrt{p})$.

For any variance vector $\sigma$, partition the variables into two groups $G_1$ and $G_2$. $G_1$ contains variables with variance at least $\frac{1}{\eps} p$, while $G_2$ contains all other variables with small variance. For any set $S_j$, let $s_j=\sum_{i\in S_j}\sigma_i^2$. The set's contribution to the total objective can be bounded by 

\begin{eqnarray*}
\E\max_{i\in S_j} X_i&\leq&\E\max\left(0,\max_{i\in S_j\cap G_1} X_i,\max_{i\in S_j\cap G_2} X_i\right)\\
&\leq&\sum_{i\in S_j\cap G_1}\max(0,X_i)+\E\max\left(0,\max_{i\in S_j\cap G_2}X_i\right)\\
&=&\sum_{i\in S_j\cap G_1}\max(0,X_i)+\sqrt{s_j}\E\max\left(0,\max_{i\in S_j\cap G_2}\frac{X_i}{\sqrt{s_j}}\right)\\
&\leq&\sum_{i\in S_j\cap G_1}\max(0,X_i)+\sqrt{s_j}\cdot O\left(\sqrt{\frac{\eps p}{s_j}\ln\frac{s_j}{\eps p}}\right)\\
&=&\sqrt{\frac{2}{\pi}}\sum_{i\in S_j\cap G_1}\sigma_i + \sqrt{\eps p}\cdot O\left(\sqrt{\ln\frac{s_j}{\eps p}}\right),
\end{eqnarray*}
here the last inequality is by Lemma~\ref{lem:eps-contribution}. If we sum up the total contribution for all sets $S_j$, notice that with high probability every variable $X_i$ is only in $O(pm)$ sets, and $s_j=O(p)$. Thus with high probability
\begin{eqnarray*}
& &\E\sum_{j\in[m]}\max_{i\in S_j} X_i\\
&\leq&\sqrt{\frac{2}{\pi}}\sum_{i\in G_1}\sigma_i\sum_{j:S_j\ni i}1+m\cdot O(\sqrt{\eps p\ln\frac{1}{\eps}})\\
&=&\sum_{i\in G_1}\sigma_i\cdot O(pm)+O(m\sqrt{\eps p\ln\frac{1}{\eps}})\\
&\leq&O(pm)\sqrt{\left(\sum_{i\in G_1}\sigma_i^2\right)}\cdot\sqrt{\left(\sum_{i\in G_1}1^2\right)}+O(m\sqrt{\eps p\ln\frac{1}{\eps}})\\
&\leq&O(pm)\cdot 1 \cdot \sqrt{\frac{\eps}{p}}+O(m\sqrt{\eps p\ln\frac{1}{\eps}})\\
&=&O(m\sqrt{p\eps\ln\frac{1}{\eps}}).
\end{eqnarray*}
This means that such an allocation only leads to $O(\sqrt{\eps\ln\frac{1}{\eps}})$ fraction of the optimal objective (which is $O(m\sqrt{p})$), thus is negligible when $\eps$ is small. 

\end{proof}

\subsection{Instance generated by arbitrary graph}
\label{subsec:multiple}
In this section, we provide an efficient approximation algorithm for \genb with competitive ratio $O(\log n)$. 
\thmmultiplesets*

\begin{proof}[Proof of Theorem~\ref{thm:polylog-approx}]

We will iteratively simplify the problem and reduce the search space of the optimal variance vector, each time only losing a small approximation factor.

Firstly, we can remove all sets $S_j$ with size 1. This is because no matter what variance we allocate to the variable $X_i$ in the set, its contribution to the total objective $\E\max_{i\in S_j}X_i=\mu_i$ does not change. Therefore, if after removing all singleton sets we get a competitive ratio $\alpha$ to the new instance, the algorithm gets $\alpha$ competitive ratio to the original instance.

Next, we will transform our problem to remove the dependency on $\mu_i$ and only lose a constant factor in the competitive ratio. For each 
$i\in [n]$, let $X'_i\sim\N(0,\sigma_i^2)$ be a zero-mean normal variable with the same variance as $X_i$. For any set $S_j$, by Lemma~\ref{lem:max0} we have
$$\E\max_{i\in S_j}X_i=\E\max_{i\in S_j}(\mu_i+X'_i)\leq \E\max_{i\in S_j}\mu_i+\E\max_{i\in S_j}\max(0,X'_i)\leq \E\max_{i\in S_j}\mu_i+2\E\max_{i\in S_j}\max(0,X_i).$$

Notice that $\E\max_{i\in S_j}\mu_i$ is an obtainable objective via setting $\sigma_i=0$ for every variable $X_i$; for any $\sigma$ vector that maximizes $\E\max_{i\in S_j}\max(0,X_i)$, using the same variance allocation only leads to a larger objective in the original instance with $\mu_i\geq 0$. Thus solving the problem with each variable having zero mean leads to a 3-approximation of the original problem. In the following proof, we just assume $\mu_i=0$ for every $i$.

Next, we will bound the search range of $\sigma$. By Lemma~\ref{lem:var2approx}, considering $\sigma_i^2$ being powers of $4$ only loses a factor of $2$ compared to the optimal solution. Thus we will only search for $\sigma$ with each $\sigma_i=4^{-k}$ or 0 for some integer $k$.

For any such $\sigma$, partition the variables $X_1,\cdots,X_n$ to the following $O(\log n)$ groups. Let $K=\log_2 n$. For each $k\leq K$, let $G_k$ contains all variables with variance $4^{-k}$. Let $G_{K+1}$ contains all other variables with variance smaller than $\frac{1}{n^2}$. Then for any set $S_j$,
\begin{eqnarray}
\E\sum_{j}\max_{i\in S_j}X_i&=&\E\sum_{j}\max\left(\max_{X_i\in S_j\cap G_0}X_i,\max_{X_i\in S_j\cap G_1}X_i,\cdots,\max_{X_i\in S_j\cap G_{K+1}}X_i\right) \nonumber\\
&\leq&\sum_{k=0}^{K}\sum_{j}\E\max(0,\max_{X_i\in S_j\cap G_k}X_i)+\sum_{j}\E\max(0,\max_{X_i\in S_j\cap G_{K+1}}X_i)\nonumber\\
&\leq&\sum_{k=0}^{K}\sum_{j}\E\max(0,\max_{X_i\in S_j\cap G_k}X_i)+\sum_{j}O\left(\frac{1}{n}\sqrt{\ln n}\right), \label{eqn:kcontribution}
\end{eqnarray}
here the last inequality is by Lemma~\ref{lem:eps-contribution}. When setting all variables to have variance $\frac{1}{n}$, for each set $S_j$, $\E\max_{i\in S_j}X_i=\Omega(\frac{1}{\sqrt{n}})=\omega\left(\frac{1}{n}\sqrt{\ln n}\right)$, thus the contribution from variables in group $G_{K+1}$ in \eqref{eqn:kcontribution} is negligible and we can ignore them.
Notice that when setting all variables to have variance either $4^{-k}$ or 0, the optimal objective is at least $\sum_{j}\E\max(0,\max_{X_i\in S_j\cap G_k}X_i)$ from this specific $\sigma$.\footnote{For every set $S_j$, if $S_j$ contains variables from only 1 group, the objective we get is actually $\E\max_{X_i\in S_j\cap G_k}X_i$ instead of $\E\max(0,\max_{X_i\in S_j\cap G_k}X_i)$.} Thus by solving the problem where all variables have the same variance $2^{-k}$, the obtained objective is at least $O(\frac{1}{\log n})$ fraction of the optimal objective for one of $k\leq \log_2 n$.

\begin{algorithm}[htb]
\caption{An $O(\log n)$-approx algorithm distributing variances to variables in multiple sets}
\begin{algorithmic}
    \STATE {\bfseries Input:} Means $(\mu_1,\cdots,\mu_n)$, Sets $S_1,\cdots,S_m$
    \STATE Initialize current best objective $M \gets 0$;
    \STATE Initialize current best $(\hatsigma_1,\cdots,\hatsigma_n)\gets(0,\cdots,0)$;
    \STATE Remove all sets with $|S_j|=1$;
    \FOR{$k\gets 0\ \textrm{ to }\log_2 n$}
        \STATE $(\sigma_1,\cdots,\sigma_n)\gets (0,\cdots,0)$;
    \WHILE{$\#$ variables with variance $4^{-k}$ $<\min(4^{k},n)$}
    \STATE Find the variable $i$ with $\sigma_i=0$ such that after setting $\sigma_i=4^{-k}$, $\obj(\boldsigma)$ is maximized;
    \STATE $\sigma_i\gets4^{-k}$;
    \ENDWHILE
  \IF{$\obj(\boldsigma)>M$}
    \STATE  $M\gets \obj(\boldsigma)$;
    \STATE  $\hatboldsigma \gets \boldsigma$;
    \ENDIF
    \ENDFOR
   \STATE {\bfseries Output:} $(\hatsigma_1,\cdots,\hatsigma_n)$
\end{algorithmic}

\end{algorithm}

Now we want to solve the following problem: find the set $S\subseteq [n]$ of variables with $|S|\leq 4^k$, such that when all variables $X_i\in S$ have variance $\sigma_i^2=4^{-k}$ and all variables not in $S$ have variance 0, the total objective is maximized. By Lemma~\ref{lem:submodularset} this is a submodular maximization problem with a cardinality constraint, and a greedy solution obtains $1-\frac{1}{e}$ fraction of the optimal solution \cite{nemhauser1978analysis}. The complete $O(\log n)$-approx algorithm is specified in Algorithm~\ref{alg:general-log-approx}.

\end{proof}

\begin{restatable}{lemma}{lemmax}\label{lem:max0}
For $n\geq 2$ independent normally distributed variables $X_1,X_2,\cdots,X_n$ with non-negative means $\mu_1,\cdots,\mu_n\geq0$,
$$\E\max(X_1,\cdots,X_n)\geq(1-2^{1-n})\E\max(0,X_1,\cdots,X_n).$$
\end{restatable}
\begin{proof}
We consider the following way to generate each $X_i\sim\N(0,\sigma_i)$. Firstly, a positive value $Y_i$ is generated from a half-normal distribution with variance $\sigma_i^2$. In other words, $Y_i$ is the absolute value of a normal variable with mean 0 and variance $\sigma_i^2$. Next, a sign $Z_i$ is generated uniformly from $\{-1,1\}$, and $X_i$ is determined by $\mu_i+Y_iZ_i$. For any fixed $Y_1,\cdots,Y_n$ we show that
\begin{eqnarray*}
\E_Z\max(\mu_1+Y_1Z_1,\cdots,\mu_n+Y_nZ_n)\geq \E_Z\max(0,\mu_1+Y_1Z_1,\cdots,\mu_n+Y_nZ_n).
\end{eqnarray*}
Suppose that $k=\arg\max_i (\mu_i+Y_i)$. Notice that when $Z_k=1$, which happens with probability $\frac{1}{2}$, $\max (\mu_i+Y_iZ_i)=\max(0,\max (\mu_i+Y_iZ_i))=\mu_k+Y_k$. When $Z_1=Z_2=\cdots=Z_n=-1$, which happens with probability $\frac{1}{2^n}$, let $t=\max (\mu_i+Y_iZ_i)$, then $t\geq -\mu_k-Y_k$. In other cases, $\max (\mu_i+Y_iZ_i)=\max(0,\max(\mu_i+Y_iZ_i))\geq 0$, and let $W\geq 0$ be the conditional expectation of $\max(\mu_i+Y_iZ_i)$ in this case. Then we have
$$\E\max (\mu_i+Y_iZ_i)=\frac{1}{2}(\mu_k+Y_k)+\frac{1}{2^n}t+\left(\frac{1}{2}-\frac{1}{2^n}\right)W,$$
and
$$\E\max(0,\max (\mu_i+Y_iZ_i)) =\frac{1}{2}(\mu_k+Y_k)+\frac{1}{2^n}\cdot\max(0,t)+\left(\frac{1}{2}-\frac{1}{2^n}\right)W.$$
As $t\geq -\mu_k-Y_k$, 
\[\frac{\E\max(0,\max (\mu_i+Y_iZ_i))}{\E\max (\mu_i+Y_iZ_i)}\leq \frac{\frac{1}{2}}{\frac{1}{2}-\frac{1}{2^n}}=\frac{2^{n-1}}{2^{n-1}-1}\] for $n\geq 2$.
Then the original lemma holds by taking the expectation of the above inequality over all possible $Y$.
\end{proof}

\varapprox*

\begin{proof}[Proof of Lemma~\ref{lem:var2approx}]
For each $i\in[n]$, we couple the generation of $X_i$ and $X'_i$ as follows. Firstly, a positive value $Y_i$ is generated from a half-normal distribution with variance $\sigma_i^2$. Next, a sign $Z_i$ is generated uniformly from $\{-1,1\}$. Then let $X_i=Y_iZ_i$, and $X'_i=\alpha_iY_iZ_i$ where $\alpha_i=\frac{\sigma_i'}{\sigma_i}$. Denote $Y'_i=\alpha_iY_i$. As $X'_i$ has the same distribution as $\alpha_iX_i$, the above procedure correctly generates $X_i\sim\N(0,\sigma_i^2)$ and $X'_i\sim\N(0,\sigma_i^2)$.

For fixed $Y=(Y_1,\cdots,Y_n)$, we study the relationship between $\E\max X'_i$ and $\E\max X'_i$. Let $Y^{(k)}$ denote the $k$th largest element in $Y$. Notice that when the largest element in $Y$ has sign $Z_i=1$, $\max_iX_i=Y^{(1)}$, and this happens with probability $\frac{1}{2}$. When $Y^{(1)}$ has sign $-1$ and $Y^{(2)}$ has sign $+1$, $\max_iX_i=Y^{(2)}$ and this happens with probability $\frac{1}{4}$. Similarly, for every $k\leq n$, $\max_iX_i=Y^{(k)}$ with probability $2^{-k}$. When all $Z_i$ are $-1$, $\max_iX_i=-Y^{(n)}$, and this happens with probability $2^{-n}$. Thus
$$\E\max_iX_i=\sum_{k=1}^{n}2^{-k}Y^{(k)}+2^{-n}(-Y^{(n)})=\sum_{k=1}^{n-1}2^{-k}Y^{(k)}.$$
Similarly, 
$$\E\max_iX'_i=\sum_{k=1}^{n-1}2^{-k}Y'^{(k)}.$$ 
For $Y=(Y_1,\cdots,Y_n)$ and $Y'=(\alpha_1Y_1,\alpha_2Y_2,\cdots,\alpha_nY_n)$ with $1\leq \alpha_i\leq 2$ for every $i\in[n]$, $Y^{(k)}\leq Y'^{(k)}\leq 2Y^{(k)}$ for every $k\in[n]$. Thus for any fixed $Y$,
$$\E\max_iX_i=\sum_{k=1}^{n-1}2^{-k}Y^{(k)}\leq \E\max_iX'_i=\sum_{k=1}^{n-1}2^{-k}Y'^{(k)}\leq 2\E\max_iX_i.$$
The original theorem holds by taking the expectation of the above inequality over all possible $Y$.
\end{proof}

\submodularset*
\begin{proof}[Proof of Lemma~\ref{lem:submodularset}]
It suffices to prove the lemma for $\sigma_0=1$, as the expected maximum with $X_i\sim\N(0,\sigma_0^{2})$ is just $\sigma_0$ fraction of the expected maximum with $X_i\sim\N(0,1)$.

For any $k\in[n]$, let $g(k)=\E\max(0,X_1,X_2,\cdots,X_k)$. Then for any $S$ with $|S|=k$, $f(S)=g(k)$, except when $|S|=n$, $g(n)=\E\max(0,X_1,X_2,\cdots,X_n)>\E\max(X_1,X_2,\cdots,X_n)=f(S)$. Thus to prove $f(S)$ is submodular, we only need to prove the submodularity of $g$, or equivalently for any $k\geq 1$, $g(k+1)-g(k)<g(k)-g(k-1)$.

Let $\Phi$ be the CDF of the standard normal distribution $\N(0,1)$. Then $\max(X_1,\cdots,X_k)$ has CDF $\Phi^{k}$, and
\begin{equation*}
    g(k)=\E\max(0,X_1,X_2,\cdots,X_k)=\int_{0}^{\infty}(1-\Phi^k(x))dx.
\end{equation*}
Then for any $k$,
\begin{eqnarray*}
    g(k)-g(k-1)&=&\int_{0}^{\infty}(1-\Phi^k(x))dx-\int_{0}^{\infty}(1-\Phi^{k-1}(x))dx\\
    &=&\int_{0}^{\infty}\Phi^{k-1}(x)(1-\Phi(x))dx\\
    &\geq&\int_{0}^{\infty}\Phi^{k}(x)(1-\Phi(x))dx\\
    &=&\int_{0}^{\infty}(1-\Phi^{k+1}(x))dx-\int_{0}^{\infty}(1-\Phi^{k}(x))dx\\
    &=&g(k+1)-g(k).
\end{eqnarray*}
Thus $g$ is a submodular function, which implies that $f$ is a submodular set function.

\end{proof}

\begin{lemma}\label{lem:max3order}
For any $a,b,c\in\R$,
\[\max(a,b)+\max(a,c)\geq\max(a,b,c)+a.\]
\end{lemma}

\begin{proof}[Proof of Lemma~\ref{lem:max3order}]
As $b$ and $c$ are symmetric in the desired inequality, without loss of generality assume $b\geq c$. Discuss the order of the variables as follows.
\begin{itemize}
\item When $a\geq b\geq c$: $LHS=a+a\geq a+a=RHS$.
\item When $b\geq a\geq c$: $LHS=b+a\geq b+a=RHS$.
\item When $b\geq c\geq a$: $LHS=b+c\geq b+a=RHS$.
\end{itemize}
Thus the lemma holds for all possible orderings of $a,b,c$.
\end{proof}

\begin{lemma}\label{lem:submodular-4var-symmetric}
For any $a,b,c,d\in \R$, 
\begin{eqnarray*}
& &3\max(a,b,c,d)+\max(a,d)+\max(b,d)+\max(c,d)\\
&\leq& 2\max(a,b,d)+2\max(a,c,d)+2\max(b,c,d). 
\end{eqnarray*}
\end{lemma}

\begin{proof}[Proof of Lemma~\ref{lem:submodular-4var-symmetric}]
By applying Lemma~\ref{lem:max3order} to $\max(a,d)$, $b$ and $c$, we get
\begin{equation*}
\max(a,b,d)+\max(a,c,d)\geq\max(a,b,c,d)+\max(a,d).
\end{equation*}
In the same way we can get
\begin{equation*}
\max(a,b,d)+\max(b,c,d)\geq\max(a,b,c,d)+\max(b,d),
\end{equation*}
\begin{equation*}
\max(a,c,d)+\max(b,c,d)\geq\max(a,b,c,d)+\max(c,d).
\end{equation*}
The lemma holds by adding the above three inequalities.
\end{proof}

\thmpersetsubmodular*

\begin{proof}[Proof of Theorem~\ref{thm:per-set-submodular}]
For any variance allocation $\boldsigma$, we prove that $f_{\boldsigma}(k)=\frac{1}{\binom{n}{k}}\obj_{I_k}(\boldsigma)$ is a concave function. Then $f=\max_{\boldsigma}f_{\boldsigma}$ is also concave, as it's the maximum over a set of concave functions.

For any fixed $\boldsigma$, for simplicity let $g=f_{\boldsigma}$. To prove the concavity of $g$, we prove that for any $k\geq 3$, 
\begin{equation}\label{eqn:g-submodular}
g(k)+g(k-2)\leq 2g(k-1).
\end{equation}
Notice that $g(k)$ is the per-set contribution to the objective. We can get an equivalent definition: randomly sample $k$ elements $(i_1,\cdots,i_k)$ from $[n]$ without replacement, then 
\begin{equation*}
g(k)=\E_{(X_1,\cdots,X_n)}\E_{(i_1,\cdots,i_k)}\max(X_{i_1},X_{i_2},X_{i_3},X_{i_4}\cdots,X_{i_k}).
\end{equation*}
Such a sampling can be viewed as first sample $k-3$ elements $(i_4,\cdots,i_k)$ without replacement, then sample 3 more elements $(i_1,i_2,i_3)$ without replacement. $g(k-1)$ can be viewed as first sample $k-3$ elements $(i_4,\cdots,i_k)$ without replacement, then sample 3 more elements $(i_1,i_2,i_3)$ without replacement, then choose two variables from $X_{i_1}$, $X_{i_2}$ and $X_{i_3}$:
\begin{eqnarray*}
g(k-1)&=&\E_{(X_1,\cdots,X_n)}\E_{(i_1,\cdots,i_k)}\frac{1}{3}\bigg(\max(X_{i_1},X_{i_2},X_{i_4}\cdots,X_{i_k})\\
& &\ \ \ \ \ \ \ \ \ \ \ \ \ \ \ \ \ \ \ \ \ \ \ \ \ \ \ \ \ \ \ \ \ \ \ \ +\max(X_{i_1},X_{i_3},X_{i_4}\cdots,X_{i_k})\\
& &\ \ \ \ \ \ \ \ \ \ \ \ \ \ \ \ \ \ \ \ \ \ \ \ \ \ \ \ \ \ \ \ \ \ \ \ +\max(X_{i_2},X_{i_3},X_{i_4},\cdots,X_{i_k})\bigg).
\end{eqnarray*}
$g(k-2)$ can be viewed as first sample $k-3$ elements $(i_4,\cdots,i_k)$ without replacement, then sample 3 more elements $(i_1,i_2,i_3)$ without replacement, then choose one variable from $X_{i_1}$, $X_{i_2}$ and $X_{i_3}$:
\begin{eqnarray*}
g(k-2)&=&\E_{(X_1,\cdots,X_n)}\E_{(i_1,\cdots,i_k)}\frac{1}{3}\bigg(\max(X_{i_1},X_{i_4},\cdots,X_{i_k})\\
& &\ \ \ \ \ \ \ \ \ \ \ \ \ \ \ \ \ \ \ \ \ \ \ \ \ \ \ \ \ \ \ \ \ \ \ \ +\max(X_{i_2},X_{i_4},\cdots,X_{i_k})\\
& &\ \ \ \ \ \ \ \ \ \ \ \ \ \ \ \ \ \ \ \ \ \ \ \ \ \ \ \ \ \ \ \ \ \ \ \ +\max(X_{i_3},X_{i_4},\cdots,X_{i_k})\bigg).
\end{eqnarray*}
To prove that $g(k)+g(k-2)\leq 2g(k-1)$, it suffices to show that this inequality holds even for every realization of $(X_1,\cdots,X_n)$ and $(i_1,\cdots,i_k)$. After removing the expectation quantifier, we get 
\begin{equation*}
g(k)=\max(X_{i_1},X_{i_2},X_{i_3},X_{i_4},\cdots,X_{i_k}),
\end{equation*}
\begin{eqnarray*}
g(k-1)&=&\frac{1}{3}\bigg(\max(X_{i_1},X_{i_2},X_{i_4},\cdots,X_{i_k})+\max(X_{i_1},X_{i_3},X_{i_4},\cdots,X_{i_k})\\
& &\ \ \ \ \ \ \ +\max(X_{i_2},X_{i_3},X_{i_4},\cdots,X_{i_k})\bigg),
\end{eqnarray*}
\begin{eqnarray*}
g(k-2)&=&\frac{1}{3}\bigg(\max(X_{i_1},X_{i_4},\cdots,X_{i_k})+\max(X_{i_2},X_{i_4},\cdots,X_{i_k})\\
& &\ \ \ \ \ \ \ +\max(X_{i_3},X_{i_4},\cdots,X_{i_k})\bigg).
\end{eqnarray*}
By applying $a=X_{i_1}$, $b=X_{i_2}$, $c=X_{i_3}$ and $d=\max(X_{i_4},\cdots,X_{i_k})$, we exactly get $3g(k)+3g(k-2)\leq 6g(k-1)$. By taking the expectation over all possible realizations of $(X_1,\cdots,X_n)$ and $(i_1,\cdots,i_k)$ we have finished the proof of \eqref{eqn:g-submodular}, thus the concavity of the function.

\end{proof}

\lemcorrindgap*

\begin{proof}[Proof of Lemma~\ref{lem:corr-ind-gap}]
For any $t\geq 0$, we analyze $\Pr[\max_{i}X_i\geq t]$ and $\Pr[\max_iY_i\geq t]$. We first prove that 
\begin{equation}\label{eqn:corr-ind-per-t}
\Pr[\max_iY_i\geq t]\geq \frac{e-1}{e}\Pr[\max_{i}X_i\geq t].
\end{equation}

For any $i\in[n]$, let $q_i=\Pr[Y_i\geq t]$. Notice that the marginal distributions $X_i$ and $Y_i$ are exactly the same normal distributions with identical mean $\mu_i$ and variance $\Sigma_{ii}$, we have $\Pr[X_i\geq t]$ is also $q_i$. Then for correlated variables $X_1,\cdots,X_n$, $\Pr[\max_iX_i\geq t]\leq \sum_{i}q_i$; for independent variables $Y_1,\cdots,Y_n$, 
\begin{eqnarray*}
\Pr[\max_iY_i\geq t]=1-\prod_{i}(1-q_i)\geq 1-\prod_i e^{-q_i}=1-e^{-\sum_i q_i}.
\end{eqnarray*}
Let $q=\sum_i q_i$. When $q\geq 1$,
\begin{eqnarray*}
\frac{\Pr[\max_iY_i\geq t]}{\Pr[\max_iX_i\geq t]}\geq \Pr[\max_iY_i\geq t]\geq 1-e^{-q}\geq 1-e^{-1}.
\end{eqnarray*}
When $q<1$,
\begin{eqnarray*}
\frac{\Pr[\max_iY_i\geq t]}{\Pr[\max_iX_i\geq t]}\geq \frac{1-e^{-q}}{q}.
\end{eqnarray*}
Let $g(q)=\frac{1-e^{-q}}{q}$. As $g'(q)=\frac{e^{-q}q-1+e^{-q}}{q^2}<0$ for $0<q<1$, we have $g$ is a decreasing function on $(0,1)$. Thus $g(q)>g(1)=1-e^{-1}$ for $0<q<1$. This finishes the proof of \eqref{eqn:corr-ind-per-t} for every $t\geq 0$.

Now we are ready to prove the original lemma using \eqref{eqn:corr-ind-per-t}. Actually,
\begin{eqnarray*}
\E\max(X_1,\cdots,X_n)&\leq& \E\max(0,X_1,\cdots,X_n)\\
&=&\int_{0}^{\infty}\Pr[\max_iX_i\geq t]dt\\
&\leq&\int_{0}^{\infty}\frac{e}{e-1}\Pr[\max_i Y_i\geq t]dt\\
&=&\frac{e}{e-1}\E\max(0,Y_1,\cdots,Y_n)\\
&\leq&\frac{2e}{e-1}\E\max(Y_1,\cdots,Y_n).
\end{eqnarray*}
Here the last inequality is from Lemma~\ref{lem:max0} for $n\geq 2$.
This finishes the proof of the lemma.

\end{proof}

\end{document}

%% file: macros_icml.tex
\usepackage{thmtools}
\usepackage{thm-restate}

\newcommand\R{\mathbb R}

\newcommand\opt{\mathsf{opt}}

\newcommand{\N}{\mathcal{N}}

\newcommand{\eps}{\epsilon}

\newcommand{\E}{\mathbb{E}}
\renewcommand{\opt}{\textsc{OPT}}
\newcommand{\obj}{\textsc{OBJ}}

\usepackage{color}
\newcounter{note}[section]

\newcommand{\tr}{\textrm{tr}}

\usepackage{bm}
\newcommand{\boldsigma}{\bm{\sigma}}
\newcommand{\hatboldsigma}{\hat{\bm{\sigma}}}
\newcommand{\hatsigma}{\hat{\sigma}}